\documentclass{article}

\usepackage{microtype}
\usepackage{graphicx}
\usepackage{subcaption}
\usepackage{booktabs} 

\usepackage{hyperref}

\usepackage[preprint]{icml2026}

\usepackage{amsmath}
\usepackage{amssymb}
\usepackage{mathtools}
\usepackage{amsthm}
\usepackage{listings}

\usepackage[capitalize,noabbrev]{cleveref}

\theoremstyle{plain}
\newtheorem{theorem}{Theorem}[section]
\newtheorem{proposition}[theorem]{Proposition}
\newtheorem{lemma}[theorem]{Lemma}

\theoremstyle{definition}
\newtheorem{definition}[theorem]{Definition}
\newtheorem{assumption}[theorem]{Assumption}
\theoremstyle{remark}

\usepackage[textsize=tiny]{todonotes}

\usepackage{bm}
\usepackage{bbm}

\DeclareMathOperator{\tr}{tr}

\newcommand{\R}{\mathbbm{R}}
\renewcommand{\L}{\mathcal{L}}

\newcommand{\vct}{\bm} 

\icmltitlerunning{Product-Stability: Provable Convergence for Gradient Descent on the Edge of Stability}

\begin{document}

\twocolumn[
  \icmltitle{Product-Stability: Provable Convergence for Gradient Descent on the Edge of Stability}



  \icmlsetsymbol{equal}{*}

  \begin{icmlauthorlist}
    \icmlauthor{Eric Gan}{indp}
  \end{icmlauthorlist}

  \icmlaffiliation{indp}{Independent Researcher}

  \icmlcorrespondingauthor{Eric Gan}{egan8@ucla.edu}

  \icmlkeywords{Machine Learning, ICML, Edge of Stability}

  \vskip 0.3in
]



\printAffiliationsAndNotice{}  

\begin{abstract}

Empirically, modern deep learning training often occurs at the Edge of Stability (EoS), where the sharpness of the loss exceeds the threshold below which classical convergence analysis applies. 
Despite recent progress, existing theoretical explanations of EoS either rely on restrictive assumptions or focus on specific squared-loss–type objectives. In this work, we introduce and study a structural property of loss functions that we term product-stability. 
We show that for losses with product-stable minima, gradient descent applied to objectives of the form $(x,y) \mapsto l(xy)$ can provably converge to the local minimum even when training in the EoS regime. 
This framework substantially generalizes prior results and applies to a broad class of losses, including binary cross entropy. 
Using bifurcation diagrams, we characterize the resulting training dynamics, explain the emergence of stable oscillations, and precisely quantify the sharpness at convergence. 
Together, our results offer a principled explanation for stable EoS training for a wider class of loss functions.

\end{abstract}

\section{Introduction} \label{sec:intro}

In classical analysis of gradient descent (GD), convergence is usually proved under the condition that the learning rate $\eta$ is less than $\frac{2}{\lambda}$, where $\lambda$, known as the sharpness, denotes the maximum eigenvalue of the Hessian. Under this assumption, the loss decreases monotonically as gradient descent converges to the local minimum. However, recent empirical observations have shown that modern machine learning models do not operate in this regime. In particular, \citet{cohen2022gradient} provided extensive evidence that during training the sharpness rises to above $\frac{2}{\eta}$, violating the classical assumption, and then oscillates around $\frac{2}{\eta}$. Meanwhile, the loss behaves non-monotonically, yet consistently decreases in the long run. 
This phenomenon, termed \emph{Edge of Stability (EoS)}, indicates that new analysis incorporating higher order information is necessary to capture the training dynamics of gradient descent in practice.

As a result, the surprising EoS phenomenon has attracted significant attention. One line of work constructs minimalist examples in which EoS training provably happens and the training dynamics can be thoroughly studied \citep{chen2023edge, ahn2023learning, zhu2023understanding, song2023trajectory}. However, existing works either require nonstandard assumptions or apply to specific functions modeled after the squared loss. Thus, rigorous analysis of EoS training dynamics for classification losses, such as the cross entropy loss, remains open.

In this work, we analyze a quantity which we call the \emph{product-stability}. This quantity was first discovered in \citet{chen2023edge}, where it was used to prove the existence of two step fixed points of GD updates. These are pairs of points $(x_1, x_2)$ such that the GD iterates cycle between them, namely applying one step of gradient descent maps $x_1$ to $x_2$ and vice versa. We uncover a much richer story, proving that for loss functions $l$ with product-stable minima, minimizing $l$ with respect to the parameter product $(x,y) \mapsto l(xy)$ via gradient descent can converge to a stationary point when starting in a region where $\lambda > \frac{2}{\eta}$. The notion of product stability captures a broad class of functions, including the binary cross entropy loss, and generalizes multiple prior results. Our contributions are as follows:
\begin{itemize}
    \item We define the notion of product-stability and use it to prove convergence in the EoS regime for functions of the form $l(xy)$ with product-stable minima.
    \item We use the notion of bifurcation diagrams \citep{song2023trajectory} to present a general pattern in the training dynamics which results in the aforementioned convergence and precisely quantify the sharpness at convergence.
    \item We show how the notion of product-stability captures practical functions such as binary cross entropy and unifies multiple previous assumptions in the literature under a single framework.
    \item We explore how the notion of product stability can be applied to deep models trained on real-world datasets.
\end{itemize}

\section{Related Work}

\textbf{Sharpness.} In the literature, low sharpness and flat minima have been associated with better generalization properties \citep{hochreiter1997flat}. And several works have verified an empirical correlation between sharpness and generalization \citep{keskar2017largebatch, jiang2019fantastic}. Moreover, the SAM algorithm showed that explicitly regularizing the sharpness can improve generalization \citep{foret2021sharpness}.

\textbf{Edge of Stability.}
Since the Edge of Stability behavior was observed in \citet{cohen2022gradient}, many works have focused on discovering the theoretical underpinnings behind this phenomenon. \citet{damian2023selfstabilization} argue that the third order term in the Taylor expansion contributes to the stabilization of the sharpness. \citet{arora2022understanding} explore the EoS phenomenon as a result of deterministic flow along the manifold of minimizers for special loss functions of the form $\sqrt{L}$ or for normalized gradient descent.  \citet{wang2022large,liang2025gradient} prove EoS convergence for quadratic losses of the form $(xy - a)^2$. \citet{ma2022quadratic} explore a multiscale loss structure with subquadratic property. 

Closely related to our work are several results that provide rigorous analysis of training dynamics in minimalist settings. \citet{zhu2023understanding} analyze loss of the form $(x^2 y^2 - 1)^2$, where they show a sharpness adaptivity property which is not present in the vanilla quadratic loss. \citet{ahn2023learning, song2023trajectory} prove EoS convergence for losses of the form $l(xy)$, assuming that $l$ is convex, even, Lipschitz, and satisfies a certain subquadratic property. \citet{wang2023good} also analyze EoS convergence of $l(xy)$ for some specific forms of $l$ based on a concept called degree of regularity. Our work generalizes these results by using a much less restrictive assumption based around the notion of product-stability. \citet{chen2023edge} use the quantity that we call product-stability to prove the existence of two-step fixed points for GD in the EoS regime, but do not explore further. Concurrent work by \citet{mulayoff2026stability} generalizes product-stability to the multivariate case.

\textbf{Large Learning Rate.} A related line of work studies the training dynamics of gradient descent using large learning rate. \citet{wu2023implicit, wu2024large, zhang2025minimax} study gradient descent on logistic regression for linearly separable data.
In this case, the minimizer is at infinity, where the loss has vanishing sharpness. This means that (at least in this setting) EoS is a transient phenomenon as the sharpness will eventually be small enough to leave the EoS regime. Similarly, \citet{cai2024large} analyze an Edge of Stability phase and a stability phase for two layer nonlinear networks with logistic loss. These settings are different from our work, where we show convergence while the sharpness stays very close to $\frac{2}{\eta}$. \citet{andriushchenko2023sgd, even2023sgd} analyze the implicit regularization of large step sizes in diagonal linear networks. However, these works also do not relate the learning rate and sharpness via the EoS threshold $\eta \lambda = 2$.

\section{Preliminaries}
Let $l: \R \to \R$ be a loss function. We focus on the objective $\L: \R^2 \to \R$ given by
\begin{equation} \label{eq:product_loss}
    \L(x, y) := l(xy)
\end{equation}
We use $l', l'', l^{(3)}, l^{(4)}$ to denote the first through fourth derivatives of $l$, respectively, and use $\nabla^n$ to denote the $n$-th derivative tensor for multivariate functions. Gradient descent with learning rate $\eta$ defines the update rules
\begin{align} 
    x_{t+1} &= x_t - \eta \frac{\partial \L}{\partial x} = x_t - \eta l'(x_t y_t) y_t \label{eq:gd_update_x} \\
    y_{t+1} &= y_t - \eta \frac{\partial \L}{\partial y} = y_t - \eta l'(x_t y_t) x_t \label{eq:gd_update_y}
\end{align}

Denote $z(x,y) = xy, s(x,y) = x^2 + y^2$, where we drop the dependence on $x$ and $y$ when clear. Our analysis focuses on dynamics in the $(s,z)$ phase space. Interpreting the $(x,y)$ parameterization as a two layer, single neuron linear model, $z$ is the output of the model and the input to the loss function. It is not hard to check that if $l$ has a minimum at $z_*$, then $\L$ has a manifold of minima along the hyperbola $z(x,y) = z_*$. The quantity $s$ is intrinsically related to the sharpness of $\L$ by the following lemma:
\begin{lemma} \label{thm:sharpness_formula}
    The trace of the Hessian of $\L$ is given by $\tr(\nabla^2 \L) = l''(z) s$. If $l'(z) = 0$, then the sharpness of $\L$ is also given by  $\lambda =  \tr( \nabla^2 \L) = l''(z) s$.
\end{lemma}

In particular, for a given value of $z$, $s$ can be viewed as a scaling factor for the sharpness based on the specific $(x,y)$ parameterization. As a result, we call $s$ the \emph{sharpness multiplier}. Moreover, observe that for any fixed $z$ we can have arbitrarily large $s$ by rescaling $x$ and $y$. In particular, there exist minima with arbitrarily large sharpness along the minimizing manifold. \footnote{\citet{dinh2017sharpminimageneralizedeep} observe that this property holds for any neural network with positive homogeneity.} 

We start our analysis in the EoS regime assuming that the GD iterates are near a minimum with sharpness slightly greater than $\frac{2}{\eta}$. Classical analysis shows that even when initialized arbitrarily close to the optimum, gradient descent will not converge to the sharp minima, instead driving iterates farther away. For a broad class of functions, we show that this is not the end of the story.

\section{Convergence on the Edge of Stability} \label{sec:main_result}

Our analysis centers around a quantity which we term the \emph{product-stability}:
\begin{definition}
    Let $f: \R \to \R$ have continuous fifth derivatives. The \emph{product-stability} of $f$ at $z \in \R$ is defined as
    \begin{equation} \label{eq:def_product_stability}
        \alpha_f(z) = 3( f^{(3)}(z))^2 - f^{(4)} (z) f''(z)
    \end{equation}
    If $\alpha_f(z) > 0$, then we say $f$ is \emph{product-stable} at $z$.
\end{definition}

\textbf{Intuition.} The product stability captures higher order information about the sharpness of $f$ around $z$. Positive values indicate stability, while negative values indicate instability.  At a local minimum where $f''(z) > 0$, the above definition implies that a negative fourth derivative makes $f$ more stable. Indeed, a negative fourth derivative corresponds to a negative quadratic term in the Taylor expansion of $f''$. Excluding all other products, this makes the sharpness at $z$ greater than the sharpness in a neighborhood around $z$. This can be related to certain ``subquadratic'' assumptions from previous work \citep{ahn2023learning, ma2022quadratic, song2023trajectory}. In addition, our definition also reveals that a large magnitude third derivative also contributes to stability. This intuitively can be thought of as a precise quantification of a more general effect observed in \citep{damian2023selfstabilization}, where the third order term tends to decrease sharpness when running gradient descent.

\subsection{Two-step Iterates}

\citet{chen2023edge} used the notion of product stability to prove that there exist fixed points of the two step GD update when the learning rate is slightly larger than $\frac{2}{\lambda}$. Here we state a slightly stronger version of their theorem in which we prove that the two step iterates converge to these fixed points:

\begin{theorem} \label{thm:two_step}
    Let $l \in \mathcal{C}^5$ have a local minimum at $z_*$ with $l''(z_*) > 0, \alpha_l(z_*) > 0$. Then there exists $\sigma, \tau > 0$ such that for all $\eta \in (\frac{2}{l''(z_*)}, \frac{2 + \tau}{l''(z_*)}]$, there exists $z_\eta^- < z_* < z_\eta^+$ such that for GD iterates given by
    \begin{equation} \label{eq:direct_gd_update}
        a_{t+1} = a_t - \eta l'(a_t)
    \end{equation}
    with initialization, 
    \begin{equation*}
        0 < |a_0 - z_*| \leq \sigma
    \end{equation*}
    one of the sequences $\{a_{2t}\}, \{a_{2t+1}\}$ converges to $z_\eta^-$, and the other converges to $z_\eta^+$.
\end{theorem}

\textbf{Proof Outline.}
Consider the change induced by a two step gradient update
\begin{equation*}
    D_\eta(a) = l'(a) + l'(a - \eta l'(a))
\end{equation*}
A Taylor expansion of $l$ around $z_*$ gives the approximation
\begin{align} \label{eq:two_step_taylor}
    D_\eta(a) &\approx (2 - \eta l''(z_*))l''(z_*) (a - z_*) \\
    \nonumber
    &+ \frac{1}{2}l^{(3)}(z_*) (1 - \eta l''(z_*)) (2 - \eta l''(z_*)) (a - z_*)^2 \\
    \nonumber
    &+ \Bigg[ \frac{1}{6} l^{(4)}(z_*) (1 - \eta l''(z_*)) (1 + (1 - \eta l''(z_*))^2) \\
    \nonumber
    &- \frac{1}{2} \eta (l^{(3)}(z_*))^2 (1 - \eta l''(z_*)) \Bigg] (a - z_*)^3
\end{align}
When $a$ is very close to $z_*$, the first order term dominates. Since $\eta l''(z_*) > 2$, the coefficient of the first term is negative. This means that the iterates move away from $z_*$, and corresponds to the traditional analysis that $z_*$ is unstable. However, notice that both the first and second order term have product $(2 - \eta l''(z_*))$. Since we assume $\eta$ is not too much larger than $\frac{2}{l''(z_0)}$, this term is small. So there exists a region where the third order term dominates. The notion of product stability is constructed precisely to ensure that the third coefficient is positive, and a simple intermediate value argument shows that there exists a point with $D_\eta(z) = 0$, namely a fixed point of the two step update. Further analysis shows that these points are unique, stable minima and the landscape is not too sharp ($|D_\eta'(z)| < \frac{2}{\eta}$), which suffices to prove that two step GD converges to these points. 

\subsection{Convergence under Reparameterization}

While the above theorem guarantees that the GD iterates will not diverge, we still do not obtain true convergence as the iterates will oscillate around the minimum with period two. Our main result proves that true convergence can be achieved simply by reparameterizing the loss in the form $l(xy)$:

\begin{theorem} \label{thm:main_convergence}
    Let $l \in \mathcal{C}^5$ have a local minimum at $z_*$ with $l''(z_*) > 0, \alpha_l(z_*) > 0$. Then there exists $\eta_0, \sigma, \tau > 0$ such that for all $0 < \eta \leq \eta_0$, minimizing the objective
    \begin{equation*}
        \L(x, y) = l(xy)
    \end{equation*}
    via gradient descent with learning rate $\eta$ and initialization
    \begin{align}
        |z(x_0,y_0) - z_*| &\leq \sigma, \\
        2 \leq \eta l''(z_*) s(x_0,y_0) &\leq 2 + \tau \label{eq:eos_init}
    \end{align}
    converges.
\end{theorem}

In the remainder of this section, we outline a simple nonconstructive proof of the above theorem. In \cref{sec:training_dynamics}, we present a detailed analysis of the training dynamics that result in convergence. In \cref{sec:discussion}, we discuss what functions satisfy our definition of product-stability and how it relates to assumptions from prior work. Complete proofs are deferred to Appendix \ref{apdx:proofs}.

\subsection{Outline of Convergence Proof} \label{sec:convergence_proof_outline}

\textbf{Bounding the iterates.}
Combining \cref{eq:gd_update_x,eq:gd_update_y} gives
\begin{equation*}
    z_{t+1} = z_t - \eta s_t l'(z_t) + \eta^2  z_t (l'(z_t))^2
\end{equation*}
In the EoS regime the last term is negligible so
\begin{equation} \label{eq:z_approx}
    z_{t+1} = z_t - \eta s_t l'(z_t)(1 + O(\eta^2)).
\end{equation}
That is, $z_t$ approximates running gradient descent directly on $l$ with adaptive learning rate $\gamma_t = \eta s_t$. Note that the EoS initialization \cref{eq:eos_init} implies that $\eta s_0$ is slightly larger than $\frac{2}{l''(z_0)}$. By controlling the change in $s_t$, we can leverage \cref{thm:two_step} to guarantee that $z_t$ remains in a given bounded region.

\textbf{From boundedness to convergence.}
Again using Equations \ref{eq:gd_update_x} and \ref{eq:gd_update_y}, we have
\begin{equation*}
    x_{t+1}^2 - y_{t+1}^2 = (x_{t}^2 - y_{t}^2)(1 - \eta^2 (l'(z_t))^2)
\end{equation*}
Since $z_t$ remains bounded, so does $l'(z_t)$, and hence for sufficiently small $\eta$, we have $0 \leq 1 - \eta^2 (l'(z_t))^2 \leq 1$. Therefore the above defines a convergent series. This implies that either $x_{t}^2 - y_{t}^2 \to 0$ or $l'(z_t) \to 0$. But for sufficiently small $\eta$ the EoS initialization (\cref{eq:eos_init}) implies that one of $x_0, y_0$ is very large in magnitude and the other is very small, and we show that GD cannot close this gap. Hence $l'(z_t) \to 0$. Combined with the boundedness of $z_t$ and the fact that $z_*$ is an isolated minimum in $\R$, this is sufficient to imply convergence.

\begin{figure}[tbp] 
  \centering
  \includegraphics[width=0.4\textwidth]{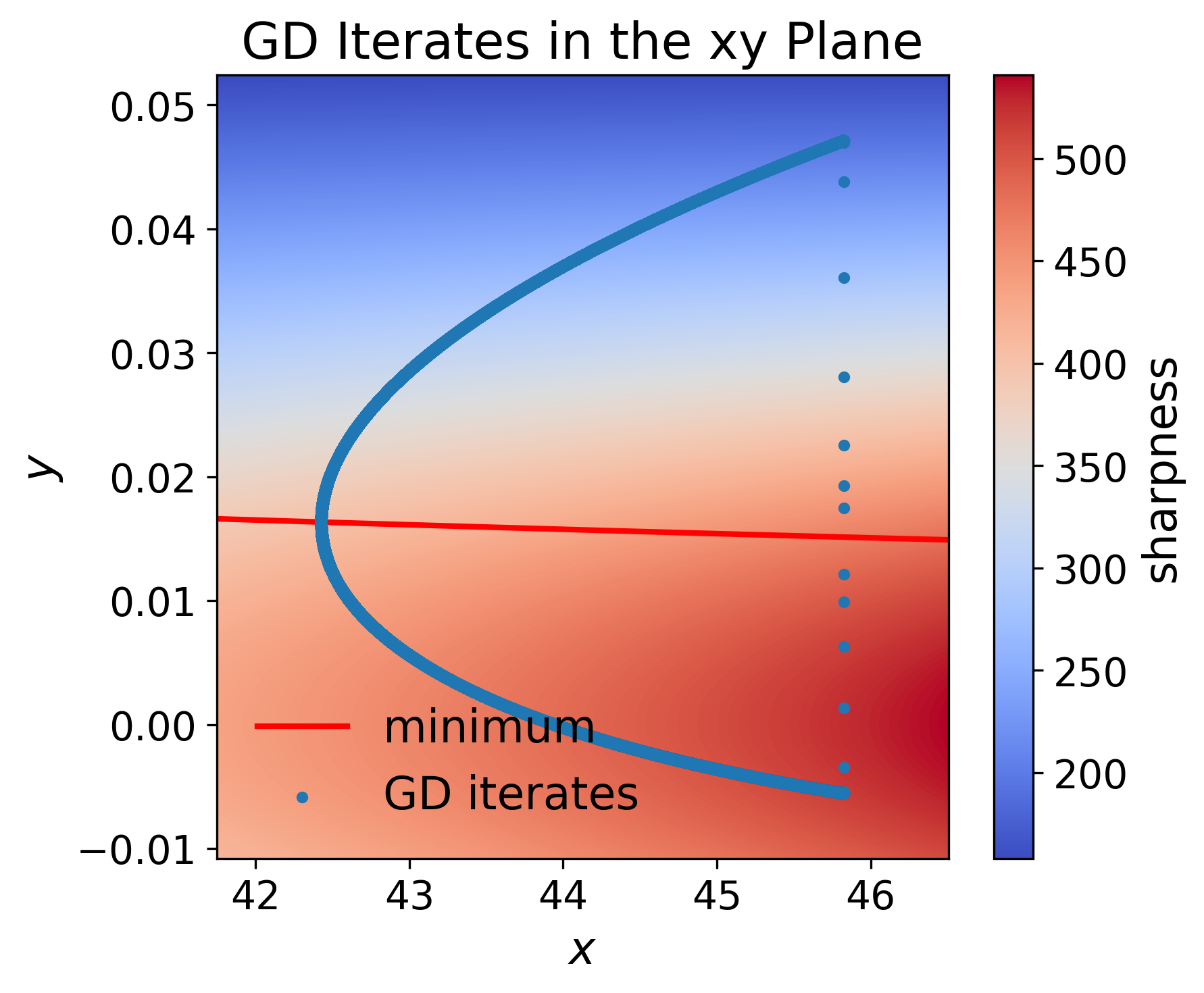}
  \caption{EoS Dynamics in the xy Plane. Iterates start on the right near a high sharpness minima. They quickly diverge away from the sharp minima before drifting towards a flatter minima on the left.}
  \label{fig:xy}
\end{figure}
\section{Training Dynamics} \label{sec:training_dynamics}

\begin{figure*}[t]
    \centering

    \begin{subfigure}{0.27\textwidth}
        \centering
        \includegraphics[width=\linewidth]{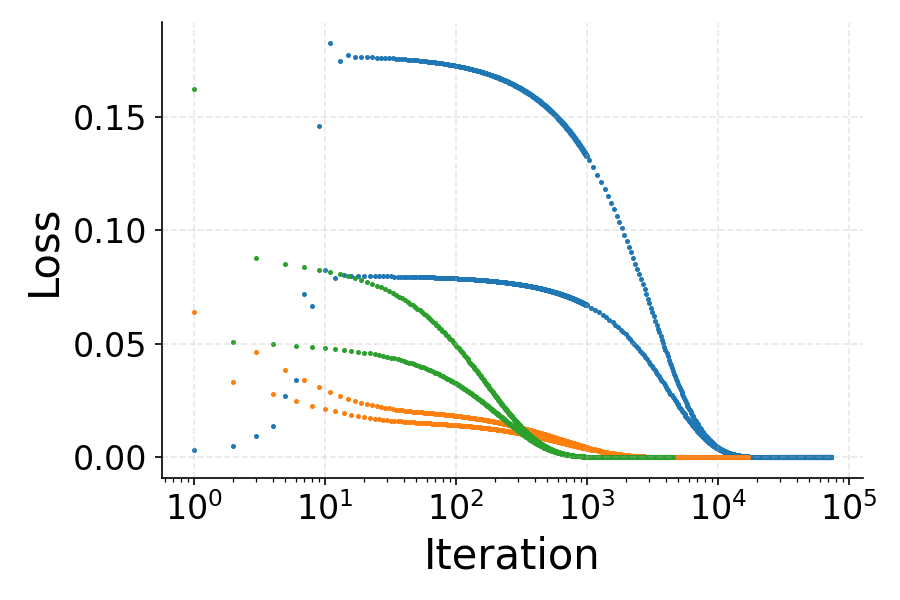}
        \caption{Loss}
    \end{subfigure}
    \hfill
    \begin{subfigure}{0.27\textwidth}
        \centering
        \includegraphics[width=\linewidth]{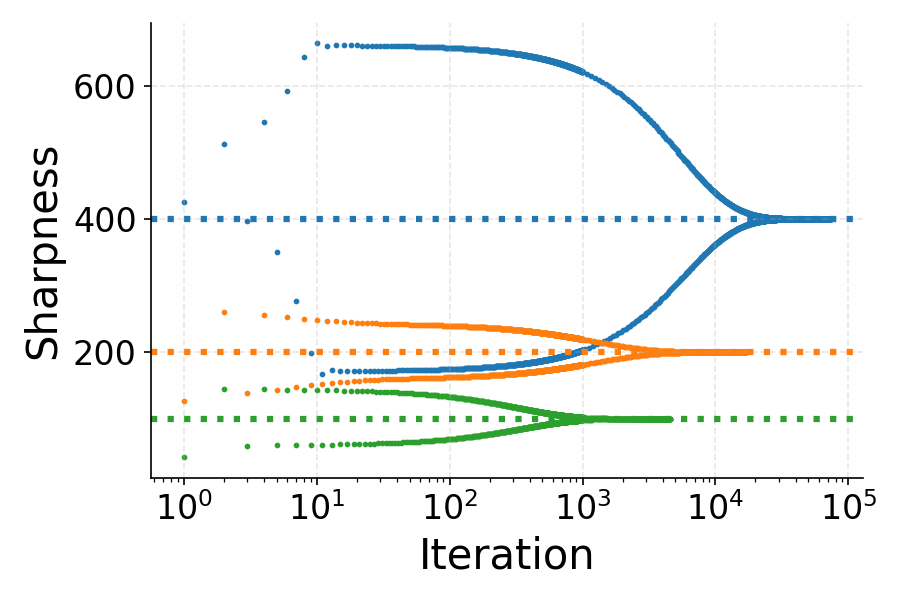}
        \caption{Sharpness}
    \end{subfigure}
    \hfill
    \begin{subfigure}{0.44\textwidth}
        \centering
        \includegraphics[width=\linewidth]{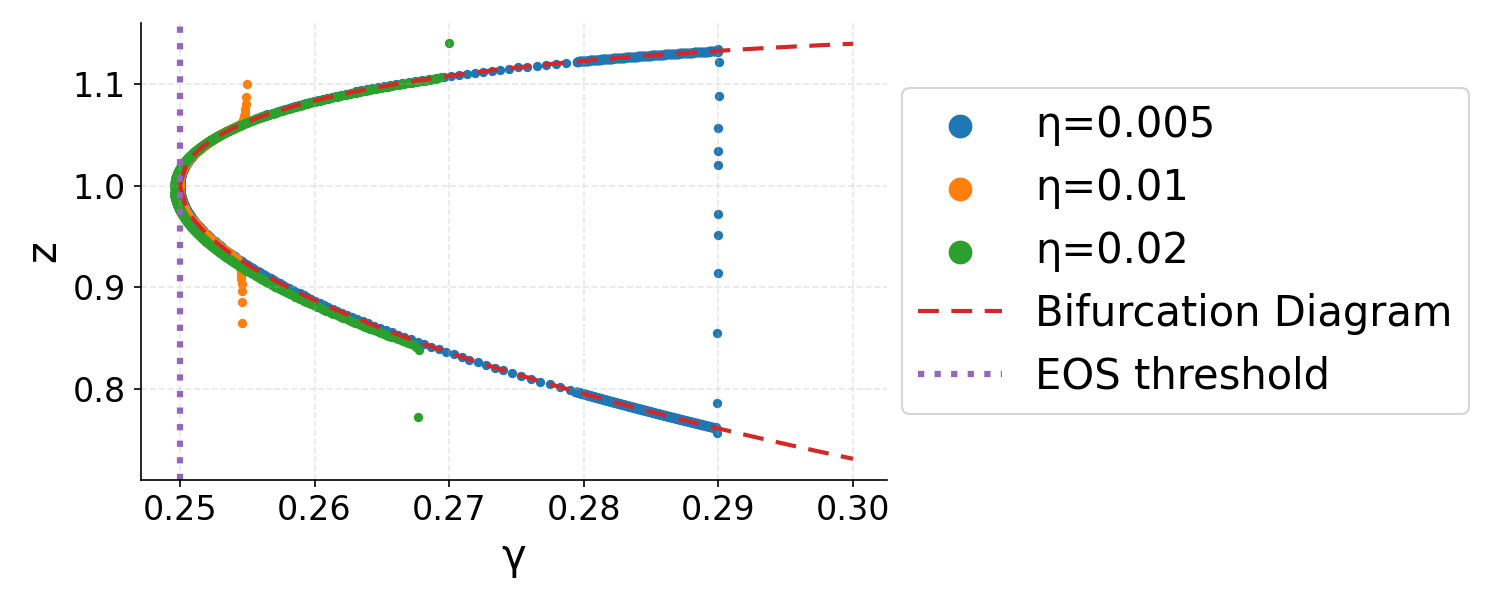}
        \caption{Training Dynamics}
    \end{subfigure}

    \captionsetup{justification=centering}
    \caption*{$\text{MLSq}_{1,2}$ (\cref{eq:mlsq_definition})}

    \vspace{1em}

    \begin{subfigure}{0.27\textwidth}
        \centering
        \includegraphics[width=\linewidth]{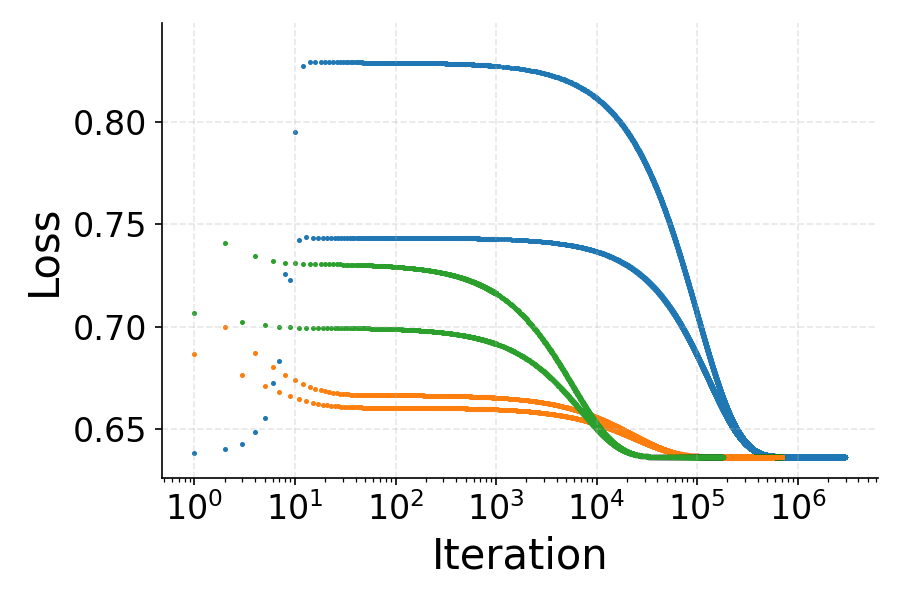}
        \caption{Loss}
    \end{subfigure}
    \hfill
    \begin{subfigure}{0.27\textwidth}
        \centering
        \includegraphics[width=\linewidth]{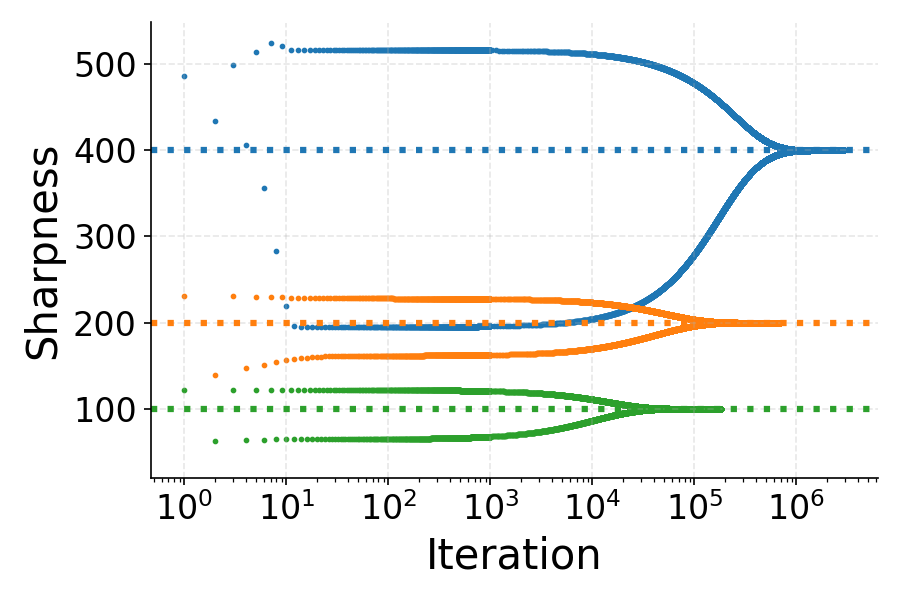}
        \caption{Sharpness}
    \end{subfigure}
    \hfill
    \begin{subfigure}{0.44\textwidth}
        \centering
        \includegraphics[width=\linewidth]{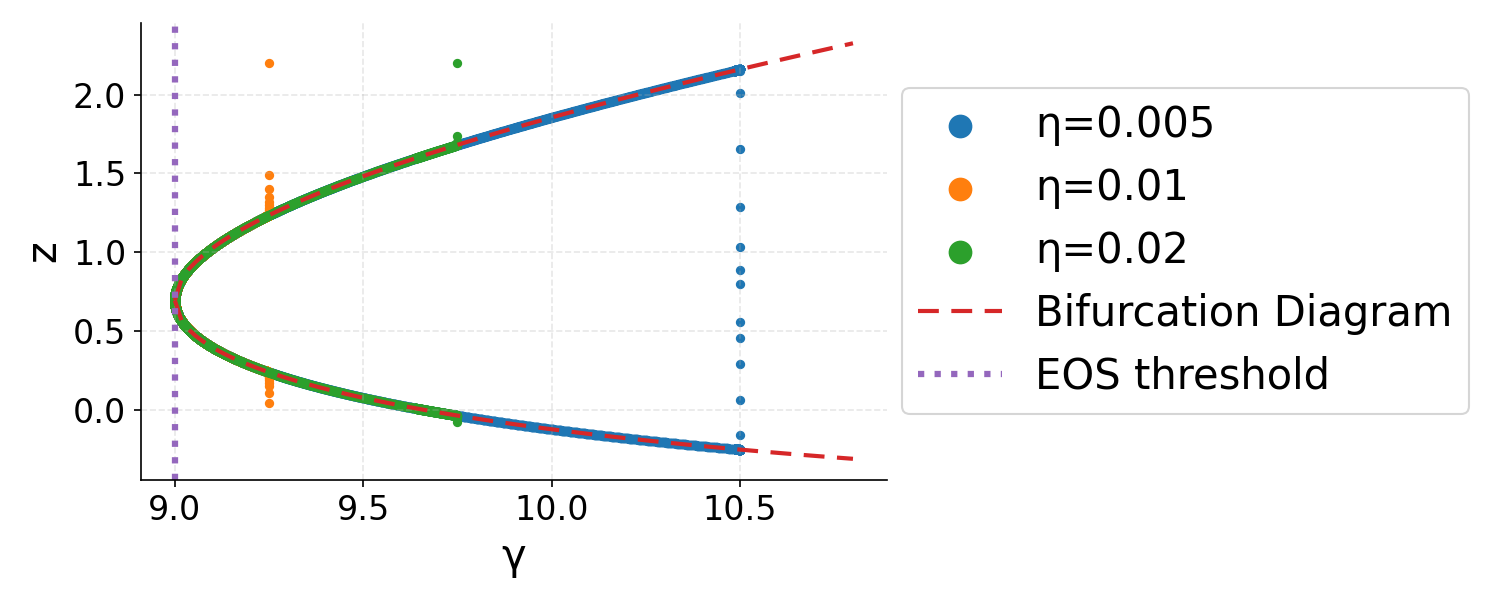}
        \caption{Training Dynamics}
    \end{subfigure}

    \caption*{$\text{BCE}_{\frac{2}{3}}$ (\cref{eq:bce_definition})}

    \captionsetup{justification=justified}
    \caption{Training dynamics of gradient descent with binary cross entropy loss (Lemma \ref{thm:bce}) and multilayer squared loss (Lemma \ref{thm:mlsq}), under the assumptions of \cref{thm:final_sharpness}. \textbf{a)} shows the training loss, which oscillates between two tracks corresponding to each branch of the period 2 bifurcation diagram. \textbf{b)} shows the sharpness, which converges to just under the EoS threshold $\lambda = \frac{2}{\eta}$. \textbf{c)} shows the dynamics in the $(\gamma, z)$ phase space. For a variety of learning rates and initializations, one can observe Phase I, where the iterates approach the bifurcation diagram (dashed line), and Phase II, where the iterates follow the bifurcation diagram towards the EoS threshold $\gamma = \frac{2}{l''(z_*)}$ (dotted line). } \label{fig:eos_training_dynamics}
\end{figure*}

While the previous section gives a clean convergence proof, the proof is nonconstructive and does not elucidate the training dynamics which leads to convergence. Here, we uncover a general mechanism of how gradient descent converges around product-stable minima on the edge of stability. Our analysis is based on the notion of bifurcation diagrams \citep{song2023trajectory}.

\subsection{Bifurcation Diagrams}
For $\eta \in (\frac{2}{l''(z_*)}, \frac{2 + \tau}{l''(z_*)}]$, we can use \cref{thm:two_step} to define maps $\eta \mapsto z_{\eta}^+$ and $\eta \mapsto z_{\eta}^-$ from a learning rate $\eta$ to the two-step fixed points of GD on $l(z)$ with the given learning rate with $z_{\eta}^- < z_* < z_{\eta}^+$. This is the (period-2) \emph{bifurcation diagram} and denote its branches by $Z^+(\eta), Z^-(\eta)$. We have shown that GD when trained on $l$ with fixed learning rate slightly greater than $\frac{2}{l''(z_*)}$, will settle into a state where it oscillates between corresponding points on the diagram. 


Here we present some basic properties of the bifurcation diagram:
\begin{lemma} \label{thm:bifurcation_properties}
    With $l$ as in \cref{thm:two_step}, there exists $\tau > 0$ such that for each branch in the period-2 bifurcation diagram $Z^{\pm}(\eta)$ defined on $(\frac{2}{l''(z_*)}, \frac{2 + \tau}{l''(z_*)}]$, the following are true
    \begin{enumerate}
        \item $Z^\pm$ have continuous fifth derivatives.
        \item $Z^\pm$ are monotonic.
        \item $Z^\pm \to z_*$ as $\eta \to \frac{2}{l''(z_*)}$.
    \end{enumerate}
\end{lemma}

We can also define an inverse $\hat{Z}$ which takes a given $z$ in the neighborhood of $z_*$ and maps it to the learning rate $\eta$ for which it is a fixed point. That is, $\hat{Z}$ satisfies
\begin{equation*}
    l'(z) + l'(z - \hat{Z}(z) l'(z)) = 0
\end{equation*}

We illustrate the bifurcation diagram in \cref{fig:eos_training_dynamics}.

\subsection{Training Dynamics of the Factored Model}

We are now ready to analyze the training dynamics on the original objective $\L$. Following the previous section, our analysis focuses on two step updates.

\textbf{Approaching the Bifurcation Diagram.}
Recall \cref{eq:z_approx}, which shows that $z_t$ behaves like training with GD directly on $l$ with adaptive learning rate $\gamma_t = \eta s_t$. The two step update for $s_t$, ignoring higher orders terms, is given by
\begin{align} \label{eq:two_step_s_approx}
    s_{t+2} &\approx s_t - 4\eta(l'(z_t) + l'(z_{t+1})) z_t \\
    \nonumber
    &+ \eta^2 \left((l'(z_t))^2 +4 l'(z_t) l'(z_{t+1}) + (l'(z_{t+1}))^2 \right) s_t
\end{align}
Now $s_t = \Theta(\frac{1}{\eta})$ while $z_t = \Theta(1)$, so $\left| \frac{s_{t+2} - s_t}{s_t} \right| = O(\eta^2)$. That is, $s_t$ does not change much over individual iterations. Therefore the effective learning rate $\gamma_t$ is roughly constant, and we can use the results from \cref{thm:two_step}. This gives rise to the first phase, where $z_t$ approaches to within $O(\eta^2)$ the bifurcation diagram while $s_t$ remains approximately constant.

\textbf{Decreasing Sharpness Multiplier.}
Once $z_t$ is close to the bifurcation diagram, we enter a new phase, where we must consider the long term drift of the sharpness multiplier. Since we are close to the bifurcation diagram defined by two step fixed points, we have $l'(z_t) + l'(z_{t+1}) = O(\eta^2) \approx 0$, so \cref{eq:two_step_s_approx} becomes
\begin{align*}
    s_{t+2} &\approx s_t (1 + 2 \eta^2 l'(z_t) l'(z_{t+1}))
\end{align*}
Due to the oscillatory nature of single steps, $l'(z_t)$ and $l'(z_{t+1})$ have opposite signs. Therefore, during this phase, the iterates track the bifurcation diagrams in the direction of decreasing $s$. This holds all the way until $z_t$ is very close to $z_*$ and $s_t$ is slightly smaller than $\frac{2}{\eta l''(z_*)}$.

\textbf{Convergence.} Now that the iterates are close to $z_*$ and the sharpness is below the EoS threshold, we show that the iterates converge linearly to the minimum.

Our formal analysis, presented in Appendix \ref{apdx_training_dynamics}, allows us to precisely characterize the sharpness at convergence.

\begin{theorem} \label{thm:final_sharpness}
    Let $l \in \mathcal{C}^5$ have a local minimum at $z_*$ with $l''(z_*) > 0, \alpha_l(z_*) > 0$. Then there exists $\eta_0, \sigma, \tau, \delta > 0$ with $\delta < \tau$ such that for all $0 < \eta \leq \eta_0$, minimizing the objective
    \begin{equation*}
        \L(x, y) = l(xy)
    \end{equation*}
    via gradient descent with learning rate $\eta$ and initialization
    \begin{align*}
        |z(x_0,y_0) - z_*| &\leq \sigma, \\
        2 + \delta \leq  \eta l''(z_*) s(x_0,y_0) &\leq 2 + \tau
    \end{align*}
    results in final sharpness
    \begin{equation*}
        \lambda = \frac{2}{\eta} - \frac{3 (l''(z_*))^4}{\alpha_l(z_*)} \eta + O(\eta^{\frac{5}{3}}).
    \end{equation*}
\end{theorem}

The theorem shows that the final sharpness is slightly less than the EoS threshold of $\frac{2}{\eta}$, demonstrating the phenomenon of sharpness adaptivity \citep{zhu2023understanding}. We note that \citet{ahn2023learning, zhu2023understanding, song2023trajectory} showed similar results under different, more restrictive assumptions. In \cref{sec:discussion}, we will find that their results are in fact special cases of the above theorem.

Compared to \cref{thm:main_convergence}, there is an additional condition requiring a gap $\delta$ between the initial sharpness multiplier $s_0$ and the EoS threshold. This gives space for the previously described EoS dynamics to take effect instead of having the iterates directly leave the EoS regime. We show that this condition is required in \cref{fig:delta_gap}.

\begin{figure}[b]
  \centering
  \includegraphics[width=0.48\textwidth]{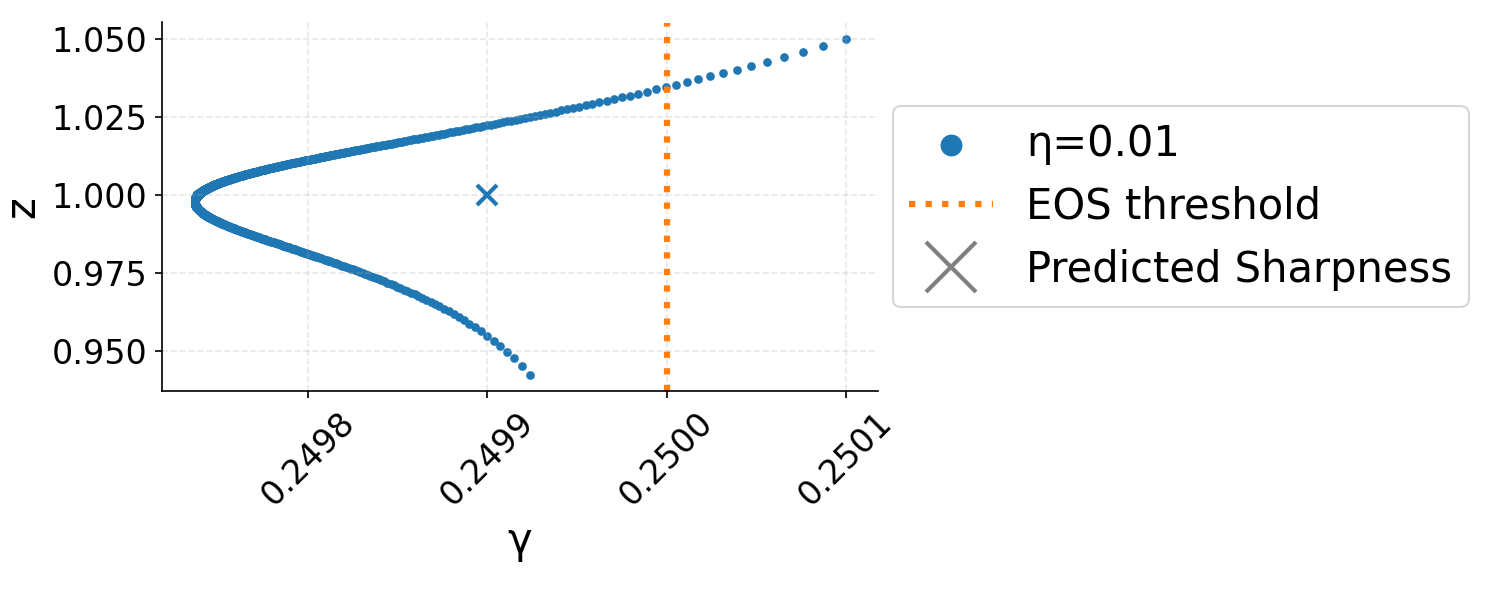}
  \caption{EoS training dynamics for $l = \text{MLSq}_{1,2}$ (\cref{eq:mlsq_definition}) when started very close to the EoS threshold. The iterates do not converge to the final sharpness predicted by \cref{thm:final_sharpness}, showing that the $\delta$ gap is required.}
  \label{fig:delta_gap}
\end{figure}

Interestingly, the correction term for the final sharpness depends directly on the product-stability. This underscores the important role of product-stability in EoS training.

\subsection{Summary} By analyzing the two-step GD dynamics, we have divided the training trajectory into three phases:
\begin{itemize}
    \item \emph{Phase I}. Iterates rapidly approach the bifurcation diagram while $s_t$ remains roughly constant.
    \item \emph{Phase II}. Iterates slowly drift along the bifurcation diagram in the direction of decreasing $s_t$ until reaching a neighborhood of the EoS minima.
    \item \emph{Phase III}. Iterates converge linearly with $z_t \to z_*$ and limiting sharpness $\lambda = \frac{2}{\eta} - \Theta(\eta)$.
\end{itemize}

\cref{fig:eos_training_dynamics} shows the training dynamics in the $(\gamma, z)$ phase space for various functions, learning rates, and initializations. We can clearly observe Phase I and Phase II, where the iterates approach the bifurcation diagram and then drift in the direction of decreasing $\gamma$. In \cref{fig:eos_training_dynamics_end}, we zoom in on the end of training, where we can see the Phase III convergence and the limiting sharpness. \cref{thm:final_sharpness} accurately predicts the sharpness at convergence.

\begin{figure*}[t] 
    \centering

    \begin{subfigure}{0.38\textwidth}
        \centering
        \includegraphics[width=\linewidth]{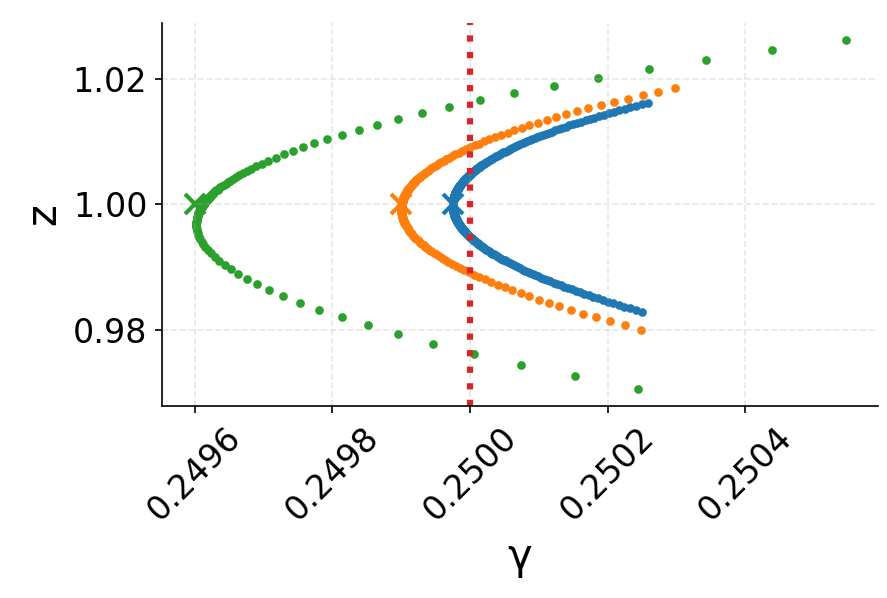}
        \caption{$\text{MLSq}_{1,2}$ (\cref{eq:mlsq_definition})}
    \end{subfigure}
    \hfill
    \begin{subfigure}{0.6\textwidth}
        \centering
        \includegraphics[width=\linewidth]{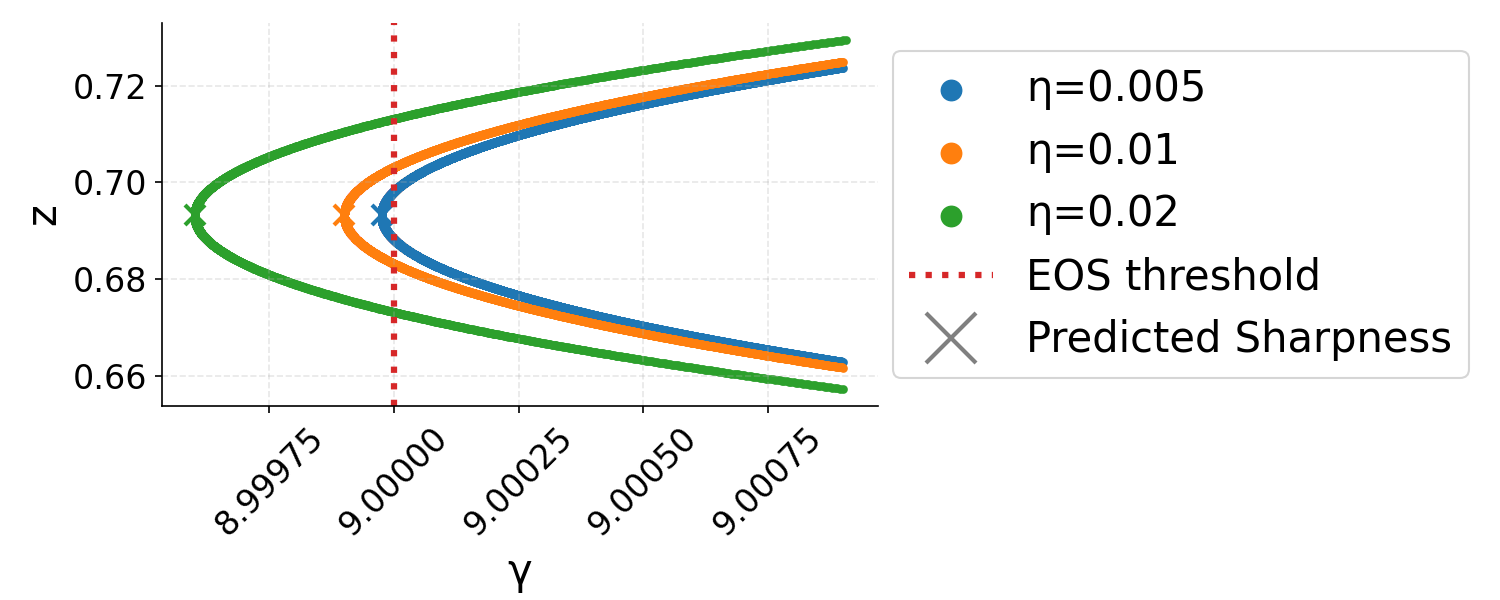}
        \caption{$\text{BCE}_{\frac{2}{3}}$ (\cref{eq:bce_definition})}
    \end{subfigure}

    \captionsetup{justification=justified}
    \caption{End of training dynamics for the runs in \cref{fig:eos_training_dynamics}. One can observe Phase III of the training dynamics, where the iterates converge to $z_*$. The limiting sharpness is just below the EoS threshold and very close to the value predicted by \cref{thm:final_sharpness}.} \label{fig:eos_training_dynamics_end}
\end{figure*}
\section{Discussion} \label{sec:discussion}
In the previous sections, we proved that minimizing functions of the form $l(xy)$ when $l$ has a product-stable minimum can converge when trained with GD on the edge of stability, and provided a detailed analysis of the dynamics that lead to convergence. In this section, we explore what types of functions $l$ have minima that satisfy the definition of product-stability. Indeed, we find that our notion of product-stability is applicable to practical losses such as binary cross entropy and generalizes assumptions from previous work on EoS training.

\subsection{Binary Cross Entropy}
The cross entropy loss and its variants are the most common classification losses used in modern applications. We define the binary cross entropy loss (BCE) by
\begin{equation} \label{eq:bce_definition}
    \text{BCE}_{q}(z) = -[q \log \sigma(z) + (1 - q) \log (1 - \sigma(z))] 
\end{equation}
where $\sigma$ is the sigmoid function and we allow soft labels $q \in [0,1]$ so that the loss can have finite minima with non-vanishing sharpness. We show that the BCE
loss is product-stable everywhere.
\begin{lemma} \label{thm:bce}
For all $q \in [0,1]$ and for all  $z \in \R, \alpha_{\text{BCE}_q}(z) > 0$.
\end{lemma}
To our knowledge, the notion of product-stability is the first which can prove EoS convergence for generic cross entropy-type losses with finite minima.
    
\subsection{Squared Loss and the Impact of Depth}
The squared loss $l(z) = (z - a)^2$ is common in regression settings and popular in theoretical analysis due to its tractability. In the EoS regime, \citet{wang2022large, liang2025gradient} prove that gradient descent on $(xy - a)^2$ converges with learning rate up to $\frac{4}{\lambda}$. One may note that the squared loss has product-stability $\alpha = 0$ everywhere, so \cref{thm:main_convergence} does not apply. 

Interestingly, \citet{zhu2023understanding} show that GD dynamics on this vanilla squared loss exhibits behavior inconsistent with empirical observation---most notably that the point to which GD converges has sharpness significantly less than $\frac{2}{\eta}$. Instead, \citet{zhu2023understanding} analyze the function $(x^2y^2 - 1)^2$, which can be thought of as squared loss with a four layer network where two of the parameters are initialized to the same value $x$ and the other two are initialized to $y$. They prove EoS convergence for this function with limiting sharpness just below the EoS threshold, and use this to argue that network depth contributes to the EoS phenomenon (note that the switch from $l(z)$ to $l(xy)$ can itself be seen as increasing the depth from 1 to 2 layers). One can generalize their construction to consider $2n$-layer networks with corresponding loss $(x^ny^n - a)^2$. We show that product-stability can prove EoS convergence for this entire class of functions \footnote{\citet{wang2023good} also analyze loss functions of this form, but under different initialization.}:
\begin{lemma} \label{thm:mlsq}
    Let $a > 0$ and define the \emph{multilayer squared loss}
    \begin{equation} \label{eq:mlsq_definition}
        \text{MLSq}_{a,n}(z) = (z^n - a)^2
    \end{equation}
    For $n \geq 2, \alpha_{\text{MLSq}_{a,n}}(a^{\frac{1}{n}}) > 0$.
\end{lemma}
\citet{zhu2023understanding} prove their result by analyzing the two step gradient dynamics in a multi-phase structure conceptually similar to our analysis in  \cref{sec:training_dynamics}. The notion of product-stability allows us to abstract the analysis and apply it to a broad class of loss functions.

\subsection{Subquadratic Losses}
\citet{ahn2023learning} proves EoS convergence for functions of the form $l(xy)$ under the following assumption:
\begin{assumption}[\citet{ahn2023learning}] \label{assumption:subquadratic} $l$ is convex, even, $1$-Lipschitz, and of class $\mathcal{C}^2$ with $l''(0) = 1$. Moreover, there exists $\beta > 1$ and $c > 0$ such that for all $z \neq 0$,
\begin{equation} \label{eq:subquadratic_assumption}
    \frac{l'(z)}{z} \leq 1 - c |z|^\beta \mathbbm{1}\{|z| \leq c\}
\end{equation}
\end{assumption} 
\cref{eq:subquadratic_assumption} implies that the loss is ``subquadratic'', namely that the sharpness decays in a neighborhood around the minimum at $z = 0$. Similar subquadratic assumptions are also made in \citet{ma2022quadratic, song2023trajectory}. Assuming sufficient differentiability, one can check that the previous assumption requires that $l^{(3)}(0) = 0$ and $l^{(4)}(0) < 0$. This immediately implies that the minimum satisfies the definition of product-stability. 
\begin{lemma}
    If $l \in \mathcal{C}^5$ satisfies Assumption \ref{assumption:subquadratic}, then $\alpha_l(0) > 0$.
\end{lemma}

Thus, the notion of product-stability can be seen as a generalization of subquadratic-ness. While subquadratic-ness requires that the third derivative be exactly $0$---a restrictive condition that is unlikely to be satisfied except in certain specially crafted scenarios---product-stability captures a much broader and more realistic class of functions.

\begin{figure*}[ht]
    \centering
    \begin{subfigure}{0.29\textwidth}
        \centering
        \includegraphics[width=\linewidth]{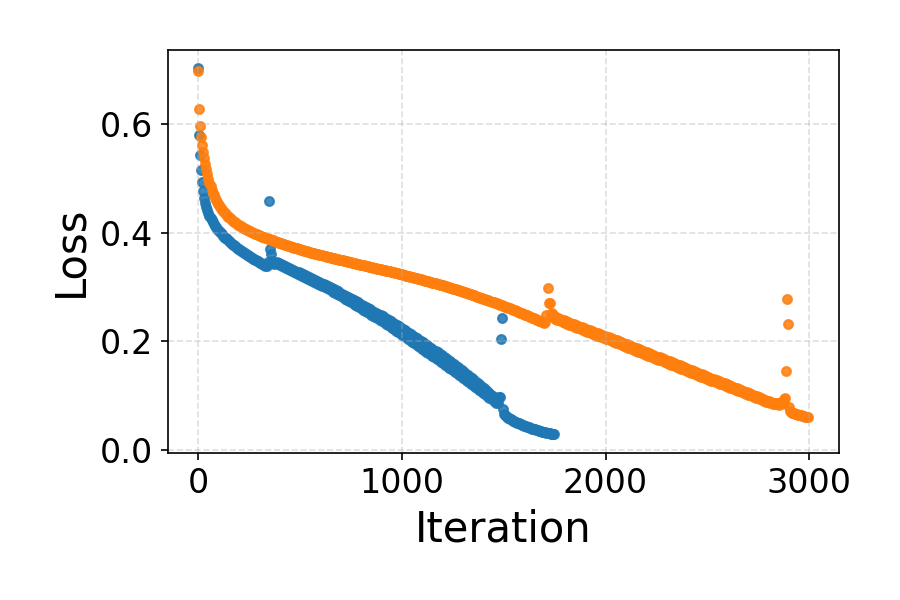}
        \caption{Train Loss}
        \label{fig:a}
    \end{subfigure}\hfill
    \begin{subfigure}{0.29\textwidth}
        \centering
        \includegraphics[width=\linewidth]{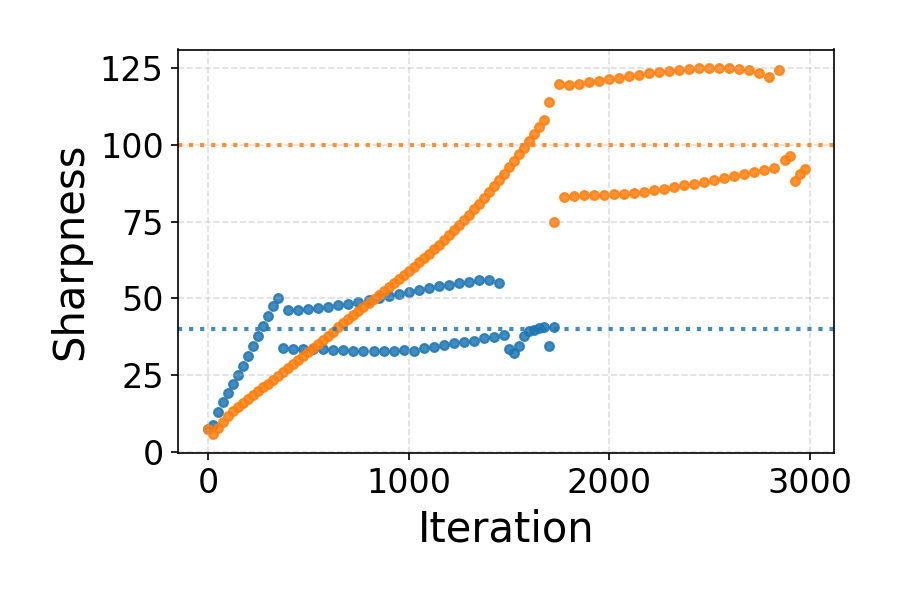}
        \caption{Sharpness}
        \label{fig:b}
    \end{subfigure}\hfill
    \begin{subfigure}{0.4\textwidth}
        \centering
        \includegraphics[width=\linewidth]{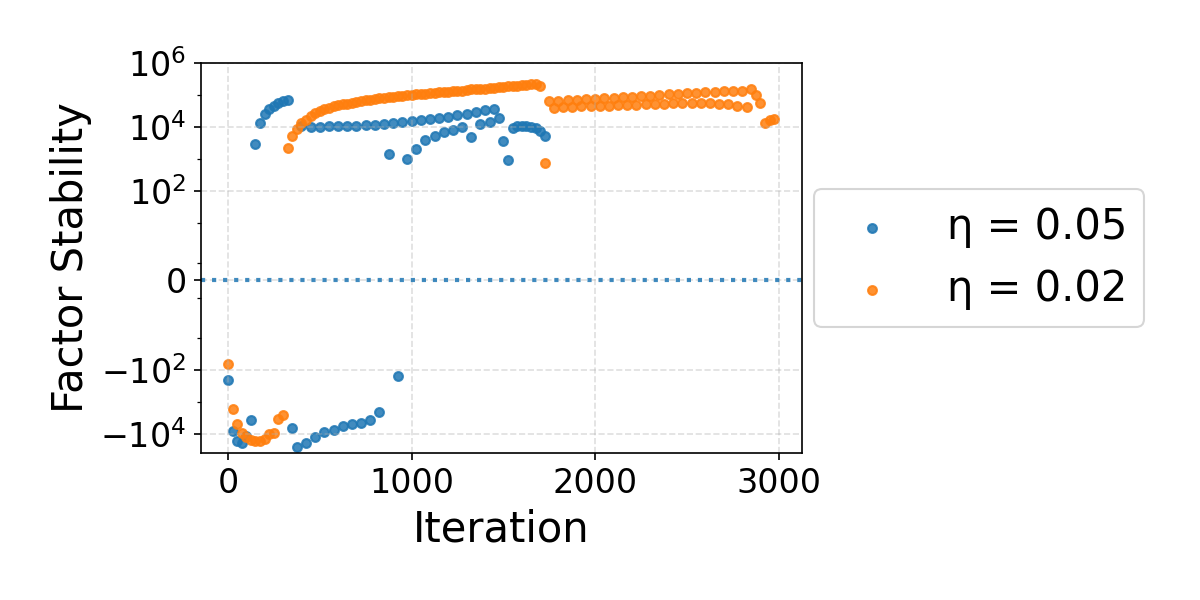}
        \caption{Product Stability}
        \label{fig:c}
    \end{subfigure}

    \caption{Training dynamics of fully-connected tanh network on CIFAR-10. \textbf{a)} shows the training loss, which consistently decreases over long timescales. The loss is also oscillating while in the EoS regime, see \cref{fig:oscillation} for zoomed in view. \textbf{b)} shows the sharpness, with the dotted lines representing the EoS threshold. Training enters the EoS regime where the sharpness oscillates around the EoS threshold. \textbf{c)} shows the product-stability calculated using directional derivatives along the direction of maximal sharpness. The product-stability is positive, indicating stability. }
    \label{fig:fc_stability}
\end{figure*}

\subsection{Degree of Regularity}
\citet{wang2023good} analyze EoS convergence based on a notion which they call degree of regularity. For the following specific class of functions with good degree of regularity, they prove convergence with limiting sharpness near the EoS threshold
\begin{equation} \label{eq:good_regularity_loss}
    l_a(z) = C_a (\log(e^{z - 1} + 1) + \log(e^{1 - z} + 1))^a
\end{equation}
where $0 < a \leq 1$ is the degree of regularity and $C_a$ is a constant. We show that this class of functions has product-stable minima.
\begin{lemma} \label{thm:dor}
Let $0 < a \leq 1$ and $l_a$ as in \cref{eq:good_regularity_loss}. Then $\alpha_{l_a}(1) > 0$.
\end{lemma}

While \citet{wang2023good} argue that degree of regularity is a \emph{qualitative} condition for EoS convergence (i.e. good regularity does not guarantee convergence), product stability defines a precise sufficient condition for convergence.
\section{Real World Models} \label{sec:real_world}

In this section, we explore how the notion of product-stability can be applied to overparameterized deep learning models on real world datasets. We train a fully connected tanh network with 2 hidden layers of width 200 on a two class, 5,000 example subset of CIFAR-10 \citep{krizhevsky2009learning}. We use learning rates $\eta = 0.05, 0.02$, leading to EoS thresholds at $\frac{2}{\eta} = 40, 100$, respectively. In both scenarios, training enters the EoS regime with sharpness greater than the EoS threshold.

Just as \citet{chen2023edge} used the product stability to prove the existence of two step fixed points for scalar functions, \citet{mulayoff2026stability} derive a generalized notion to prove the same in higher dimensions. We adopt their notion to define the \emph{multivariate product stability}. 

\begin{definition}
Let $f: \R^d \to \R$ have continuous fifth derivatives and $\vct{z} \in \R^d$. Suppose that $\nabla^2 f (\vct{z})$ is positive definite and has a unique maximal eigenvalue corresponding unit eigenvector $\vct{v}_{\max}$. The \emph{multivariate product-stability} of $f$ at $\vct{z}$ is defined as 
\begin{equation*}
    \alpha_f(\vct{z}) = \nabla^3f(\vct{z})[\vct{v_{\max}}]
^2[\vct{q}] - \nabla^4f(\vct{z}) [\vct{v_{\max}}]^4
\end{equation*}
where $\vct{q} = 3 \left[\nabla^2 f (\vct{z})\right]^{-1} \nabla^3 f(\vct{z})[\vct{v}_{\max}]^2$.
\end{definition}

In our experiments, we use the pseudoinverse to handle the scenario where the Hessian $\nabla^2 f(\vct{z})$ may be singular. Details are presented in \cref{apdx:experimental_setup}. 

In \cref{fig:fc_stability}, we observe that when the sharpness oscillates around the EoS threshold, the multivariate product-stability has a large positive value. This can help indicate why training remains stable in the EoS regime.
\section{Conclusion}

In this work, we rediscover the notion of product-stability and highlight its role in the dynamics of EoS training. We prove that a loss $l$ with product-stable minima, combined with the parameterization $l(xy)$, is sufficient to prove convergence when training with gradient descent on the edge of stability. Furthermore, we show a predictable pattern in the training dynamics of such functions which allows us to pinpoint the sharpness at convergence. The notion of product-stability generalizes previous assumptions in the literature, and paves the way for future analysis in less restrictive and more complex settings.

There are still many open questions. Since our notion of product-stability depends only on the derivatives at the minimum, our analysis is inherently local. Under what conditions a global analysis of GD dynamics can be performed remains unresolved. A consequence of this fact is that this work does not address the progressive sharpening phenomenon by which deep learning models enter the EoS regime. Fully understanding if and how complex neural network loss landscapes satisfy product-stability type conditions also remains a question.

\section*{Impact Statement}

This paper presents work whose goal is to advance the field of Machine
Learning. There are many potential societal consequences of our work, none
which we feel must be specifically highlighted here.

\bibliography{reference}
\bibliographystyle{icml2026}


\newpage
\appendix
\onecolumn

\section{Proofs} \label{apdx:proofs}

\subsection{Preliminaries}

Here we provide a glossary of shorthand notation used throughout the proofs. All of these can have subscript $t$ to denote the value at time $t$.
\begin{align*}
    z &= xy \\
    s &= x^2 + y^2 \\
    \gamma &= \eta s \\
    g &= l'(z) \\
    h_t &= g_{t+1}^2 + 4 g_{t+1} g_t + g_t
\end{align*}

We use Big-O notation to capture dependencies on $\eta$ and $|z - z_*|$ (equivalently $g$) as both quantities go to zero. Exact constants will depend on the loss function $l$ and choices of $\sigma, \tau, \delta$ from \cref{thm:main_convergence} and \cref{thm:final_sharpness} when relevant.

Now we present basic facts about the two step updates
\begin{proposition}
We have the following two step update formulas
\begin{align}
    \label{eq:two_step_z}
    z_{t+2} - z_t &= \eta^2 h_t z_t - \eta(g_{t+1} + g_t)(1 + \eta^2g_{t+1}g_t) s_t + \eta^4 g_{t+1}^2 g_t^2 z_t \\
    \label{eq:two_step_s}
    s_{t+2} - s_t &= \eta^2 h_t s_t - 4\eta(g_{t+1} + g_t)(1 + \eta^2g_{t+1}g_t) z_t + \eta^4 g_{t+1}^2 g_t^2 s_t
\end{align}
\end{proposition}

\subsection{Proof of \cref{thm:two_step}}
We start off with a few lemmas.

\begin{lemma} \label{thm:c3_lower_bound}
Let
\begin{equation*}
    c_3(\eta) = \frac{l^{(4)}(z_*)}{6}(1 - \eta  l''(z_*))(1 + (1 - \eta  l''(z_*))^2) - \frac{\eta ( l^{(3)}(z_*))^2(1 - \eta  l''(z_*))}{2}
\end{equation*}
Then there exists $\tau_0 > 0$ such that $c_3 = \min_{\frac{2}{l''(z_*)} \leq \eta \leq \frac{2+\tau_0}{l''(z_*)}}\{c_{3,\eta}\} > 0$.
\end{lemma}
\begin{proof}
If $l^{(4)}(z_*) < 0$ then both terms of $c_{3}(\eta)$ are positive and bounded away from 0. If $l^{(4)}(z_*) = 0$, then \cref{eq:def_product_stability} implies $l^{(3)}(z_*) \neq 0$, hence the second term in $c_3(\eta)$ is positive and bounded away from 0.

If $l^{(4)}(z_*) > 0$, observe that $\beta := \frac{3(l^{(3)}(z_*))^2}{l''(z_*) l^{(4)}(z_*)} > 1$. Then there exists $\tau_0 > 0$ such that $(1 +\tau_0)(1 + (1 +\tau_0)^2) = 1 + \beta$. Now
\begin{align*}
    c_{3}(\eta) &\geq  -\frac{l^{(4)}(z_*)}{6}(1 +\tau_0)(1 + (1 +\tau_0)^2) + \frac{\eta ( l^{(3)}(z_*))^2(1)}{2}  \\
    &\geq  -\frac{l^{(4)}(z_*)}{6}(1 +\tau_0)(1 + (1 +\tau_0)^2) + \frac{( l^{(3)}(z_*))^2}{l''(z_*)} \\
    &\geq -\frac{l^{(4)}(z_*)}{6}(1 + \beta) + \frac{( l^{(3)}(z_*))^2}{l''(z_*)} \\
    &= \frac{l^{(4)}(z_*)}{6} (-(2 + \beta - 1) + 2\beta) \\
    &= \beta - 1 \\
\end{align*}
This proves the claim.
\end{proof}

We now prove \cref{thm:two_step}.
\begin{proof}
We present the proof for $z_\eta^+$. A symmetric argument works for $z_\eta^-$.

Since $l''(z_*) > 0$, we can choose a bounded interval $\mathcal{I} = [z_1, z_2]$ containing $z_*$ such that $l''(z_t), l''(z_{t+1}) > 0$ for  $z_t \in \mathcal{I}$.

Let $K_1 = \frac{1}{24} \max_{z \in \mathcal{I}}{\{ | l^{(5)}(z)| \}}$, such that
\begin{align*}
     l'(z) =  l''(z_*) (z - z_*) + \frac{ l^{(3)}(z_*)}{2} (z - z_*)^2 + \frac{l^{(4)}(z_*)}{6} (z - z_*)^3 + M(z) (z - z_*)^4
\end{align*}
where $|M(z)| \leq K_1$. Then with $D_\eta(z)$ as defined in \cref{eq:two_step_taylor}, we have
\begin{align*}
    D_\eta(z) &= l'(z) + l'(z - \eta l'(z)) \\
    &= (2 - \eta l''(z_*)) l''(z_*)(z - z_*) \\
    &+ \frac{l^{(3)}(z_*)}{2}(1 - \eta l''(z_*))(2 - \eta  l''(z_*))(z - z_*)^2 \\
    &+  \left(\frac{l^{(4)}(z_*)}{6}(1 - \eta  l''(z_*))(1 + (1 - \eta  l''(z_*))^2) - \frac{\eta ( l^{(3)}(z_*))^2(1 - \eta  l''(z_*))}{2} \right)(z - z_*)^3 \\
    &+ M_2(\eta, z) (z - z_*)^4
\end{align*}
where $|M_2(\eta, z)| \leq K_2$ for some constant $K_2$. For shorthand let $c_1(\eta), c_2(\eta), c_3(\eta)$ denote the first three coefficients, respectively. We can also write 
\begin{align*}
    D_\eta'(z) &= c_1(\eta) + 2 c_2(\eta)(z - z_*) + 3c_3(\eta)(z - z_*)^2 + M_3(\eta, z) (z - z_*)^3
\end{align*}
where $|M_3(\eta, z)| \leq K_3$ for some constant $K_3$.

We claim that we can choose $\tau_0 > 0$ such that $c_3 = \min_{\frac{2}{l''(z_*)} \leq \eta \leq \frac{2+\tau_0}{l''(z_*)}}\{c_{3,\eta}\} > 0$. Indeed, if $l^{(4)}(z_*) < 0$ then both terms of $c_{3}(\eta)$ are positive. If $l^{(4)}(z_*) = 0$, then \cref{eq:def_product_stability} implies $l^{(3)}(z_*) \neq 0$, hence the second term in $c_3(\eta)$ is positive.

If $l^{(4)}(z_*) > 0$, observe that $\beta := \frac{6(l^{(3)}(z_*))^2}{l''(z_*) l^{(4)}(z_*)} > 2$. Then there exists $\tau_0 > 0$ such that $(1 +\tau_0)(1 + (1 +\tau_0)^2) = 2 + \frac{\beta - 2}{2}$. Now
\begin{align*}
    c_{3}(\eta) &\geq  -\frac{l^{(4)}(z_*)}{6}(1 +\tau_0)(1 + (1 +\tau_0)^2) + \frac{\eta ( l^{(3)}(z_*))^2(1)}{2}  \\
    &\geq  -\frac{l^{(4)}(z_*)}{6}(1 +\tau_0)(1 + (1 +\tau_0)^2) + \frac{( l^{(3)}(z_*))^2}{l''(z_*)} \\
    &\geq -\frac{l^{(4)}(z_*)}{6}(2 + \frac{\beta - 2}{2}) + \frac{( l^{(3)}(z_*))^2}{l''(z_*)} \\
    &= \frac{l^{(4)}(z_*)}{6} (-(2 + \frac{\beta - 2}{2}) + \beta) \\
    &= \frac{1}{2}(\beta - 2) \\
\end{align*}
This proves the claim.

Set $u_{\eta} = z_* + \min\left\{\left|\frac{c_{1, \eta}}{4c_{2, \eta}}\right|, \sqrt{\left|\frac{c_{1, \eta}}{4c_{3, \eta}}\right|}, \sqrt[3]{\left|\frac{c_{1, \eta}}{4K_2}\right|}\right\}$. Then $D_\eta(z) < 0$ for $z_* < z \leq u_\eta$.

Define $C_3 = \max_{\frac{2}{l''(z_*)} \leq \eta \leq \frac{2+\alpha_0}{l''(z_*)}}\{c_3(\eta)\}$.
Choose 
\begin{align*}
   z_4 &= z_* + \min\left\{\frac{c_3}{4 K_2}, \frac{c_3}{4 K_3},  \sqrt{\frac{l''(z_*)}{12 C_3}}, z_2 - z_* \right\}, \\
   \tau &= \min\left\{1, \tau_0, \frac{c_3(z_4 - z_*)^2)}{4 l''(z_*)}, \frac{c_3(z_4 - z_*))}{4 |l^{(3)}(z_*)|}\right\} 
\end{align*}

We show $D_\eta(z_4) > 0$. Indeed, for $\eta$ sufficiently close to $\frac{2}{l''(z_*)}$, we have that $c_1(\eta), c_2(\eta)$ are arbitrarily small so
\begin{align*}
    D_\eta(z^+) &= c_1(\eta) (z^+ - z_*) + c_2(\eta)(z^+ - z_*)^2 + c_3(\eta)(z^+ - z_*)^3 + M_2(\eta, z^+) (z^+ - z_*)^4 \\
    &> c_3 (z^+ - z_*)^3 - |c_1(\eta)(z^+ - z_*)| - |c_2(\eta)(z^+ - z_*)^2| - K_2 (z^+ - z_*)^4 \\
    &> 0
\end{align*}

Thus by IVT, there exists $u_\eta < z_\eta^+ < z_4$ with $D_\eta(x_\eta^+) = 0$. This proves the existence of the two-step fixed points.

Additionally, for any $z \in [z_*, z^+]$ with $D_\eta(z) = 0$, we have
\begin{align*}
    D_\eta'(z) &= c_1(\eta) + 2 c_2(\eta)(z - z_*) + 3c_3(\eta)(z - z_*)^2 + M_3(\eta, z) (z - z_*)^3 \\
    &= -c_1(\eta) + c_3(\eta)(z - z_*)^2 + M_3(\eta, z) (z - z_*)^3 - 2M_2(\eta,z)(z - z_*)^3 \\
    &> -c_1(\eta) + c_3(\eta)(z - z_*)^2 - K_3 (z - z_*)^3 - 2K_2 (z - z_*)^3 \\
    &> 0
\end{align*}
This implies two things
\begin{itemize}
    \item By the implicit function theorem, $z_\eta^+$ is uniquely defined.
    \item $z_\eta^+$ is a local minimum of the two step loss landscape.
\end{itemize}
Finally for $z_* < z \leq z^+$
\begin{align} 
    \label{eq:d_derivative_bound}
    |D_\eta'(z)| &= |c_1(\eta) + 2 c_2(\eta)(z^+ - z_*) + 3c_3(\eta)(z^+ - z_*)^2 + M_3(\eta, z) (z - z_*)^3| \\
    \nonumber
    &\leq 3c_3 (z^+ - z_*)^2 + |c_1(\eta)| + |2c_2(\eta)(z^+ - z_*)| + K_3 (z^+ - z_*)^3 \\
    \nonumber
    &\leq 4c_3 (z^+ - z_*)^2 \\
    \nonumber
    &\leq \frac{l''(z_*)}{3} \\
    \nonumber
    &\leq \frac{l''(z_*)}{2+\tau} \\
    \nonumber
    &\leq \frac{1}{\eta}
\end{align}
Thus the Descent Lemma (Lemma \ref{thm:descent_lemma}) implies convergence of the two step updates to $z_\eta^+$.

\end{proof}

\subsection{Proof of Lemma \ref{thm:bifurcation_properties}} \label{apdx:bifurcation_properties_proof}

We can prove Lemma \ref{thm:bifurcation_properties} as a corollary of the previous theorem.

\begin{proof}
Since we showed that $D_\eta'(z_\eta^+) > 0$ and $D_\eta(z)$ has continuous fifth derivative, the Implicit Function Theorem implies that $z_\eta$ has continuous fifth derivative. 

Now note that
\begin{equation*}
    \partial_\eta D_\eta(z) = -l'(z) l''(z - \eta l'(z))
\end{equation*}
Now $l''$ is always positive by assumption, and $l'(z_\eta^+) > 0$, so the implicit function theorem implies that $(z^+)'$ never changes sign, i.e. $z^+(\eta)$ is monotonic.

Finally, for $\eta$ sufficiently close to $\frac{2}{l''(z_0)}$ choose
\begin{equation*}
    z_{4,\eta} = z_* + \max\left\{2\sqrt{\left|\frac{c_1(\eta)}{c_3(\eta)}\right|}, 4 \left|\frac{c_2(\eta)}{c_3(\eta)}\right|  \right\} \leq z_4
\end{equation*}
Then $D_\eta(z_{4, \eta}) > 0$, which gives $z_* < z_\eta^+ < z_{4,\eta}$. Since $z_{4,\eta} \to z_*$ as $\eta \to \frac{2}{l''(z_0)}$,  we have $z_\eta^+ \to z_*$ by the Squeeze Theorem.
\end{proof}

\subsection{Proof of \cref{thm:main_convergence}}
We first start off with some lemmas:
\begin{lemma} \label{thm:xy_balance}
The following equation holds for all $t \geq 0$:
\begin{equation*}
    x_{t+1}^2 - y_{t+1}^2 = (x_{t}^2 - y_{t}^2)(1 - \eta^2 (l'(z_t))^2)
\end{equation*}
\end{lemma}
\begin{proof}
From \cref{eq:gd_update_x} and \cref{eq:gd_update_y}, we have
\begin{align*}
    x_{t+1}^2 &= x_t^2 - 2 \eta z_t l'(z_t) + \eta^2 (l'(z_t))^2 y_t^2 \\
    y_{t+1}^2 &= y_t^2 - 2 \eta z_t l'(z_t) + \eta^2 (l'(z_t))^2 x_t^2
\end{align*}
Adding the two equations gives the result.
\end{proof}

\begin{lemma} \label{thm:sharpness_bound}
The following equation holds for all $t \geq 0$:
\begin{equation*}
    s_{t+1}^2 - 4 z_{t+1}^2 = (s_{t}^2 - 4 z_{t}^2)(1 - \eta^2 (l'(z_t))^2)^2
\end{equation*}
\end{lemma}
\begin{proof}
Observe that
\begin{align*}
    s_t^2 - 4 z_t^2 &= (x_t^2 + y_t^2)^2 - 4 x_t y_t \\
    &= (x_t^2 - y_t^2)^2 
\end{align*}
The result follows from the previous lemma.
\end{proof}

We also need a few lemmas to extend our analysis to the regime where the sharpness is below the EoS threshold.

\begin{lemma} \label{thm:d_derivative_bound_below_eos}
Suppose that $\alpha_l(z_*) > 0$. Then there exists $\rho > 0$ such that if $\eta \in [\frac{2 - \rho}{l''(z_*)}, \frac{2}{l''(z_*)}]$, then $|D_\eta'(z)| \leq \frac{1}{\eta}$.
\end{lemma}
\begin{proof}
Observe that 
\begin{align*}
    c_3(\frac{2}{l''(z_0)}) &= -\frac{L^{(4)}(z)}{3} + \frac{(l^{(3)}(z_*))^2}{l''(z_*)} \\
    &= \frac{\alpha_l(z_*)}{l''(z_*)} \\
    &> 0.
\end{align*}
By continuity there exists $\rho_0$ such that $c_3 > \frac{\alpha_l(z_*)}{2l''(z_*)}$ on $[\frac{2 - \rho}{l''(z_*)}, \frac{2}{l''(z_*)}]$. Then a similar argument to \cref{eq:d_derivative_bound} gives the desired result.

\end{proof}

We now prove \cref{thm:main_convergence}:

\begin{proof}
Let $\tau_0, \sigma_0$ be as guaranteed in \cref{thm:two_step}, and set $\tau = \frac{\tau_0}{2}, \sigma = \sigma_0$. We can assume $\sigma$ is sufficiently small such that $l''(z) > 0$ on $[z_* - \sigma, z_* + \sigma]$. That is $l$ is convex over the interval and $z_*$ is the unique minimum.

We first show $z_t$ remains bounded. Specifically, we prove by induction that the following jointly hold:
\begin{itemize}
    \item $\eta s_t l''(z_0) \leq 2+\frac{\tau_0}{ l''(z_*)}$
    \item $|z_t - z_*| \leq \sigma$
\end{itemize}

We can assume the base cases $t = 0,1$ hold by construction. Assume the inductive hypothesis holds up to time $t$, then by Lemma \ref{thm:sharpness_bound}, for $\eta$ sufficiently small we have
\begin{align*}
    s_{t+2}^2 - 4 z_{t+2}^2 &\leq s_{0}^2 - 4 z_{0}^2  \\
    s_{t+2}^2 &\leq s_{0}^2 + 4 (z_{t+2}^2 - z_{0}^2) \\
    &\leq \left(\frac{2+\frac{\alpha}{2}}{\eta  l''(z_*)}\right)^2 + (2\sigma)^2 \\
    &\leq \left(\frac{2+\alpha}{\eta  l''(z_*)}\right)^2
\end{align*}

Denote $z^+ = z_* + \sigma$. We have $\partial_\gamma(D_\gamma(z)) = -l'(z) l''(z - \gamma l'(z))$. Since $l''$ is positive over the region and $l'(z) > 0$ for $z > z_*$, we have that $D_\gamma(z^+)$ is a decreasing function of $\gamma$. This implies
\begin{align*}
    z_t - \gamma_t D_{\gamma_t}(z_t) \leq z_t - \gamma_t D_{\frac{2+\alpha}{l''(z_*)}}(z_t)
\end{align*}

Additionally, from the proof of \cref{thm:two_step}, we have $|\partial_z  D_{\frac{2+\alpha}{l''(z_*)}}(z)| < (\frac{2+\alpha}{l''(z_*)})^{-1} < \gamma_t^{-1}$, so $z \mapsto z - \gamma_t D_{\frac{2+\alpha}{l''(z_*)}}(z)$ is increasing. Hence
\begin{align*}
    z_t - \gamma_t D_{\frac{2+\alpha}{l''(z_*)}}(z_t) \leq z^+ - \gamma_t D_{\frac{2+\alpha}{ l''(z_*)}}(z^+)
\end{align*}

It follows that for $\eta$ sufficiently small we have that
\begin{align*}
    z_{t+2} &= z_t + \eta^2(g_{t+1}^2 + 4g_{t+1}g_t + g_t^2)z_t - \eta(g_{t+1} + g_t)(1 + \eta^2g_{t+1}g_t) s_t + \eta^4 g_{t+1}^2 g_t^2 z_t \\
    &\leq z_t - \gamma_t D_{\gamma_t}(z_t) + |\eta^2(g_{t+1}^2 + 4g_{t+1}g_t + g_t^2)s_t| + |\gamma_t(g_{t+1} + g_t - D_{\gamma_t}(z_t))| \\
    &+ |\eta^3(g_{t+1}+g_t)g_{t+1}g_ts_t| + |\eta^4 g_{t+1}^2 g_t^2 s_t| \\
    &\leq  z^+ - \gamma_t \left(f_{\frac{2+\alpha}{ L''(z_*)}}(z^+) + O(\eta)\right) \\
    &\leq z^+.
\end{align*}
A symmetric argument proves that $z_{t+2} \geq z_* - \sigma$. 

This completes the induction.

Now recall from Lemma \ref{thm:xy_balance} that
\begin{equation*}
    x_{t+1}^2 - y_{t+1}^2 = (x_{t}^2 - y_{t}^2)(1 - \eta^2 (l'(z_t))^2)
\end{equation*}
For sufficiently small $\eta$, we have $0 \leq 1 - \eta^2 (l'(z_t))^2 \leq 1$. Therefore the above forms a convergent series. This implies that either $l'(z_t) \to 0$ or $x_{t}^2 - y_{t}^2 \to 0$. 

Case 1: $l'(z_t) \to 0$. Combined with the boundedness of $z_t$ and the fact that $z_*$ is an isolated minimum in $\R$, this immediately implies convergence of $z_t$ to $z_*$.

Case 2: $x_{t}^2 - y_{t}^2 \to 0$. Letting $L$ be the Lipschitz constant of $l'$ on $[z_* - \sigma, z_* + \sigma]$, for any $\epsilon > 0$ we can assume there exists a time $t_0$ where $s_t = \sqrt{(x_t^2 - y_t^2)^2 + 4 z_t^2} \leq \frac{1}{\eta L} - \epsilon$ for all $t \geq t_0$. But using \cref{eq:z_approx} and local convexity, there exists $\alpha_t$ with $0 < \alpha_t < 1$ such that
\begin{align*}
    z_{t+1} - z_* = (1 - \alpha_t) (z_t - z_*)
\end{align*}
namely $z_t$ converges monotonically.
\end{proof}

\subsection{Descent Lemma}
Here we give a formal statement of the Descent Lemma for reference.
\begin{lemma}[Descent Lemma] \label{thm:descent_lemma}
Let $f : \mathbb{R}^n \to \mathbb{R}$ be continuously differentiable, and suppose that its gradient $\nabla f$ is $L$-Lipschitz continuous, i.e.,
\[
\|\nabla f(x) - \nabla f(y)\| \leq L \|x - y\| \quad \text{for all } x,y \in \mathbb{R}^n.
\]
Then, for all $x,y \in \mathbb{R}^n$, we have
\[
f(y) \leq f(x) + \nabla f(x)^\top (y - x) + \frac{L}{2} \|y - x\|^2.
\]
\end{lemma}

\section{Analysis of Training Dynamics} \label{apdx_training_dynamics}

We start by characterizing a few useful functions for our analysis. 

\begin{lemma} \label{thm:z_hat_definition}
Define $\hat{Z}$ as the inverse the bifurcation diagrams $Z^{\pm}$, with $\hat{Z}(z_*) = \frac{2}{l''(z_*)}$. Then $\hat{Z}$ is twice continuously differentiable with
\begin{align*}
    \hat{Z}'(z_*) &= 0 \\
    \hat{Z}''(z_*) &= \frac{2 \alpha_l(z_*)}{3 (l''(z_*))^3}
\end{align*}
\end{lemma}
\begin{proof}
It suffices to check the behavior at $z_*$. Continuity of $\hat{Z}$ follows from Lemma \ref{thm:bifurcation_properties}.

By construction $\hat{Z}$ satisfies $D_{\hat{Z}(z)}(z) = 0$. Taking the implicit derivative, we find that
\begin{align*}
    \hat{Z}'(z_*) &= - \frac{c_2(\hat{Z}(z_*))}{c_1'(\hat{Z}(z_*))} = \frac{0}{(l''(z_*))^2} = 0
\end{align*}
and 
\begin{align*}
    \hat{Z}''(z_*) &= - \frac{2c_3(\hat{Z}(z_*)) + 2c_2'(\hat{Z}(z_*)) \hat{Z}'(z_*) + c_3''(\hat{Z}(z_*)) (\hat{Z}'(z_*))^2 }{c_1'(\hat{Z}(z_*))} \\
    &= - \frac{2c_3(\hat{Z}(z_*)) }{c_1'(\hat{Z}(z_*))} \\ 
    &= \frac{\frac{2 \alpha_l(z_*) }{3 l''(z_*)}}{(l''(z_*))^2} \\
    &= \frac{2 \alpha_l(z_*)}{3 (l''(z_*))^3}
\end{align*}
\end{proof}

\begin{lemma} \label{thm:phi_definition}
Letting $o(z) = z - \hat{Z}(z) l'(z)$, define
\begin{equation}
    \Phi(z) =
    \begin{cases}
    \frac{-\hat{Z}'(z) l'(z) z \left( l''(o(z)) + \frac{2}{\hat{Z}(z)}\right) + 2l'(z)}{\hat{Z}'(z) l''(o(z))}, & z \neq z_*, \\
    \frac{3 (l''(z_*))^3}{\alpha_l(z_*)},  & z = z_*.
    \end{cases}
\end{equation}
Then $\Phi$ is twice continuously differentiable. 
\end{lemma}
\begin{proof}
It suffices to check the continuity at the hole $z = z_*$. By Lemma \ref{thm:z_hat_definition}, we have $\hat{Z}(z_*) = \frac{2}{l''(z_*)}$ and
\begin{equation*}
    \hat{Z}'(z) = \frac{2 \alpha_l(z_*)}{3 (l''(z_*))^3} (z - z_*) + O(|z - z_*|^2)
\end{equation*}
Also $l''(o(z_*)) = l''(z_*)$ and $l'(z) = l''(z_*)(z - z_*) + O (|z - z_*|^2)$. It follows that
\begin{align*}
    \lim_{z \to z_*} \Phi(z) &= \frac{2 l''(z_*)(z - z_*) + O(|z - z_*|^2)}{\frac{2 \alpha_l(z_*)}{3 (l''(z_*))^3} (z - z_*) l''(z_*) + O(|z - z_*|^2)} \\
    &= \frac{3 (l''(z_*))^3}{\alpha_l(z_*)}
\end{align*}
\end{proof}

\subsection{Phase I}
The following Lemma shows that GD iterates approach the bifurcation diagram.
\begin{lemma} \label{thm:phase_1_approx}
Under the conditions of \cref{thm:final_sharpness}, there exists $T_0$ such that the following holds:
\begin{align}
    \label{eq:phase_1_approx}
    |\hat{Z}(z_{T}) - \gamma_{T}| &=  O(\eta^2) \\
    \nonumber
    \gamma_T - \frac{2}{l''(z_*)} &\geq \frac{\delta}{2}
\end{align}

\end{lemma}

\begin{proof} 
By definition of $\hat{Z}$, we have
\begin{equation*}
    l'(z_t) + l'(z_t - \hat{Z}(z_t) l'(z_t)) = 0
\end{equation*}

Therefore, by the mean value theorem there exists $\theta_t$ such that
\begin{align*}
    g_t + g_{t+1} &= -l'(z_t - \hat{Z}(z_t) g_t) + l'(z_t - \gamma_t g_t + \eta^2 g_t^2 z_t) \\
    &= l''(\theta_t)g_t (\hat{Z}(z_t) - \gamma_t) + l''(\theta_t)\eta^2 g_t^2 z_t
\end{align*}

Then from \cref{eq:two_step_z}, we have
\begin{align} \label{eq:two_step_z_approx}
    z_{t+2} - z_t &= -\gamma_t(g_t + g_{t+1}) + \eta^2 (h_t z_t - (g_t + g_{t+1})g_t g_{t+1} \gamma_t) + O(\eta^4 g_t^2) \\
    \nonumber
    &= -\gamma_t l''(\theta_t) g_t (\hat{Z}(z_t) - \gamma_t) + O(\eta^2 g_t^2) 
\end{align}

Now using the Mean Value Theorem on $\hat{Z}$ gives
\begin{align*}
    \hat{Z}(z_{t+2}) = \hat{Z}(z_t) + \hat{Z}'(\zeta_t) \left[-\gamma_t l''(\theta_t) g_t (\hat{Z}(z_t) - \gamma_t)\right] + O(\eta^2 g_t^2)
\end{align*}

On the other hand, by equation \cref{eq:two_step_s} we have 
\begin{align} \label{eq:two_step_gamma_approx}
    \gamma_{t+2} - \gamma_t &= - 4\eta^2(g_{t+1} + g_t) z_t + O(\eta^2 g_t^2) \\
    \nonumber
    &= -4 \eta^2 l''(\theta_t) g_t z_t (\hat{Z}(z_t) - \gamma_t) + O(\eta^2 g_t^2)
\end{align}

Combining the previous two equations gives
\begin{equation*}
    \hat{Z}(z_{t+2}) - \gamma_{t+2} = (1 - \hat{Z}'(\zeta_t) \gamma_t l''(\theta_t) g_t + 4 \eta^2 l''(\theta_t) g_t z_t) (\hat{Z}(z_{t}) - \gamma_t) + O(\eta^2 g_t^2)
\end{equation*}

We split into three cases:

Case 1: $\hat{Z}(z_t) + K\eta^2 < \gamma_t, |g_t| \geq M \eta^2$.

Now, looking at \cref{eq:two_step_z_approx}, we know that $-\gamma_t l''(\theta_t) g_t (\hat{Z}(z_t) - \gamma_t)$ has the same sign as $z_t - z_*$. Moreover, using Lemma \ref{thm:z_hat_definition}, observe that $\hat{Z}'(\zeta_t) \gamma_t l''(\theta_t) g_t = \Theta(g_t^2)$, so we can choose $K, M$ sufficiently large such that there exists constant $C_1$ such that
\begin{equation} \label{eq:linear_bifurcation_error}
    \hat{Z}(z_{t+2}) - \gamma_{t+2} \leq \left(1 - C_1 g_t^2 \right) (\hat{Z}(z_{t}) - \gamma_t)
\end{equation}

The second inequality implies that $|\hat{Z}(z_{t}) - \gamma_{t}|$ is monotonically decreasing, hence bounded. Moreover, \cref{eq:two_step_z_approx} now implies that $|z_{t} - z_*|$ is increasing, so $g_t^2$ does not vanish. Then there exists $T_0$ such that $|\hat{Z}(z_{t}) - \gamma_{t}| \leq K \eta^2$.

Now \cref{eq:linear_bifurcation_error} implies that
\begin{align*}
    (\hat{Z}(z_{0}) - \gamma_0) - (\hat{Z}(z_{T_0}) - \gamma_{T_0}) &= \sum_t (\hat{Z}(z_{t}) - \gamma_t) - (\hat{Z}(z_{t+2}) - \gamma_{t+2}) \\
    &\geq C_1 \sum_t g_t^2 (\hat{Z}(z_{t}) - \gamma_t) \\
    &\geq C_1 K \eta^2 \sum_t g_t^2
\end{align*}

On the other hand, using \cref{eq:two_step_gamma_approx}, there exists $C_2$ such that
\begin{align*}
    |\gamma_{T_0} - \gamma_0| & \leq \sum | \gamma_{t+2} - \gamma_t | \\
    &\leq C_2 \eta^2 [(\hat{Z}(z_{0}) - \gamma_0) - (\hat{Z}(z_{T_0}) - \gamma_{T_0})] + \sum C_3 \eta^2 g_t^2 \\
    &\leq \left(C_2 \eta^2 + \frac{C_3}{C_1 K}\right) [(\hat{Z}(z_{0}) - \gamma_0) - (\hat{Z}(z_{T_0}) - \gamma_{T_0})]
\end{align*}
Choosing $K$ sufficiently large and $\eta$ sufficiently small guarantees that this quantity is less than $\frac{\delta}{2}$, as desired.

Case 2: $\hat{Z}(z_t) - K \eta^2 > \gamma_t$.

By induction we assume that $|\gamma_t - \frac{2}{l''(z_*)}| \geq \frac{\delta}{2}$. Then $|z_t - z_*| \geq Z^{\pm} (\frac{2}{l''(z_0} + \frac{\epsilon}{2})$ so $\hat{Z}'(\zeta_t) \gamma_t l''(\theta_t) g_t$ is lower bounded by some constant $c$. It follows that \cref{eq:phase_1_approx} holds at time $T = O(\log \frac{1}{\eta})$. Now that $\gamma_{t+2} - \gamma_t = O(\eta^2)$, 
so 
\begin{align*}
    \gamma_{T} &\geq \gamma_{0} - \sum_t |\gamma_{t+2} - \gamma_t| \\
    &= \frac{2}{l''(z_0)} + \delta - O(\eta^2 \log{\frac{1}{\eta}}) \\
    &\geq \frac{2}{l''(z_0)} + \frac{\delta}{2}
\end{align*}
\end{proof}

Case 3: $|g_t| \leq M \eta^2$.

Since we are in the EoS regime, the update \cref{eq:z_approx} has a linearly unstable critical point at $z_*$. Hence the iterates diverge away from $z_*$ until we enter Case 1.

\subsection{Phase II}

Now we provide a more precise characterization of how GD iterates follow the bifurcation diagram.

\begin{lemma} \label{thm:phase_2_bound}
Under the conditions of \cref{thm:final_sharpness}, there exists $C$ such that if $g_t \geq C \eta^{2}$, then
\begin{align*}
    \hat{Z}(z_{t+2}) - \gamma_{t+2} - \eta^2 \Phi(z_{t+2}) &= (1 - \hat{Z}'(z_t) \gamma_t l''(o_t) g_t) (\hat{Z}(z_{t}) - \gamma_t + \eta^2 \Phi(z)) + O(\eta^4 g_t)
\end{align*}

\end{lemma}

\begin{proof} 
By definition of $\hat{Z}$, we have
\begin{equation*}
    l'(z_t) + l'(z_t - \hat{Z}(z_t) l'(z_t)) = 0
\end{equation*}

Therefore, we have
\begin{align*}
    g_t + g_{t+1} &= -l'(z_t - \hat{Z}(z_t) g_t) + l'(z_t - \gamma_t g_t + \eta^2 g_t^2 z_t) \\
    &= l''(o_t)(\hat{Z}(z_t) g_t - \gamma_t g_t + \eta^2 g_t^2 z_t) + O(\eta^4 g_t^2) \\
    &= O(\eta^2 g_t)
\end{align*}

Then from \cref{eq:two_step_z}, we have
\begin{align*}
    z_{t+2} - z_t &= -\gamma_t(g_t + g_{t+1}) + \eta^2 (h_t z_t - (g_t + g_{t+1})g_t g_{t+1} \gamma_t) + O(\eta^4 g_t^2) \\
    &= -\gamma_t l''(o_t)(\hat{Z}(z_t) g_t - \gamma_t g_t + \eta^2 g_t^2 z_t) - 2 \eta^2 g_t^2 z_t + O(\eta^4 g_t^2) \\
    &= -\gamma_t l''(o_t)g_t(\hat{Z}(z_t) - \gamma_t) - \eta^2 g_t^2 z_t(\gamma_t l''(o_t) + 2) + O(\eta^4 g_t^2) \\
    &= O(\eta^2 g_t)
\end{align*}
where we additionally used the simplification that $h_t = (g_t + g_{t+1})^2 + 2 g_t g_{t+1} = -2 g_t^2 + O(\eta^2 g_t^2)$.

Now a Taylor expansion of $\hat{Z}$ gives
\begin{align*}
    \hat{Z}(z_{t+2}) = \hat{Z}(z_t) + \hat{Z}'(z_t) \left[-\gamma_t l''(o_t)g_t(\hat{Z}(z_t) - \gamma_t) - \eta^2 g_t^2 z_t(\gamma_t l''(o_t) + 2)\right] + O(\eta^4 g_t^2)
\end{align*}

On the other hand, by equation \cref{eq:two_step_s} we have 
\begin{align}
    \label{eq:two_step_gamma_approx_refined}
    \gamma_{t+2} - \gamma_t &= \eta^2 (h_t \gamma_t - 4(g_{t+1} + g_t) z_t) + O(\eta^4 g_t^2) \\
    \nonumber
    &= -2 \eta^2 g_t^2 \gamma_t - 4\eta^2 l''(o_t) g_t (\hat{Z}(z_t) - \gamma_t) z_t + O(\eta^4 g_t^2)
\end{align}

Combining the previous two equations gives
\begin{equation*}
    \hat{Z}(z_{t+2}) - \gamma_{t+2} = (1 - \hat{Z}'(z_t) \gamma_t l''(o_t) g_t + 4\eta^2 l''(o_t) g_t z_t) (\hat{Z}(z_{t}) - \gamma_t) + \eta^2 g_t^2 \left(-\hat{Z}'(z_t) z_t(\gamma_t l''(o_t) + 2) + 2 \gamma_t\right) + O(\eta^4 g_t^2)
\end{equation*}

Using $g_t \geq C \eta^2$, we can simplify the former to 
\begin{equation*}
    \hat{Z}(z_{t+2}) - \gamma_{t+2} = (1 - \hat{Z}'(z_t) \gamma_t l''(o_t) g_t) (\hat{Z}(z_{t}) - \gamma_t) + \eta^2 g_t^2 \left(-\hat{Z}'(z_t) z_t(\gamma_t l''(o_t) + 2) + 2 \gamma_t\right) + O(\eta^4 g_t)
\end{equation*}
Now by Lemma \ref{thm:phi_definition} and \cref{eq:two_step_z_approx}, $\Phi(z_{t+2}) - \Phi(z_t) = O(\eta^2 g_t^2)$, and we have
\begin{align*}
    \hat{Z}(z_{t+2}) - \gamma_{t+2} - \eta^2 \Phi(z_{t+2}) &=  (1 - \hat{Z}'(z_t) \gamma_t l''(o_t) g_t) (\hat{Z}(z_{t}) - \gamma_t) + \eta^2 g_t^2 \left(-\hat{Z}'(z_t) z_t(\gamma_t l''(o_t) + 2) + 2 \gamma_t\right) \\
    &- \eta^2 \Phi(z_t) + O(\eta^4 g_t) \\
    &= (1 - \hat{Z}'(z_t) \gamma_t l''(o_t) g_t) (\hat{Z}(z_{t}) - \gamma_t + \eta^2 \Phi(z)) +O(\eta^4 g_t)
\end{align*}
where we replaced $2 = \frac{2\gamma_t}{\hat{Z}(z_t)} + O(\eta^2)$. 
\end{proof}

The following lemma characterizes the claim of decreasing sharpness during Phase II.
\begin{lemma} \label{thm:phase_2_sharpness}
Suppose that $|\hat{Z}(z_t) - \gamma_t| = O(\eta^2)$. Then there exists $C$ such that if $g_t \geq C \eta^{2}$, then
\begin{align*}
    \gamma_{t+2} &\leq \gamma_t  
\end{align*}

\end{lemma}

\begin{proof}
The statement $\gamma_{t+2} \leq \gamma_t$ follows from \cref{eq:two_step_gamma_approx_refined}, as the first term on the right hand side is negative with magnitude $\Theta(\eta^2 g_t^2)$ while the other terms have order at most $O(\eta^4 g_t)$.
\end{proof}

We now leverage Lemma \ref{thm:phase_2_bound} to get a better approximation of the GD trajectory.
\begin{lemma} \label{thm:phase_2_approx}
Under the conditions of \cref{thm:final_sharpness}, there exists $T_1$ such that the following hold
\begin{align}
    \label{eq:phase_2_approx}
    |\hat{Z}(z_{T_1}) - \gamma_{T_1} - \eta^2 \Phi(z_{T_1})| &= O(\eta^4) \\
    \nonumber
    \gamma_{T_1} - \frac{2}{l''(z_*)} &\geq \frac{\delta}{4}
\end{align}

\end{lemma}

\begin{proof}
By Lemma \ref{thm:phase_1_approx}, there exists $T_0$ such that
\begin{align*}
    |\hat{Z}(z_{T_0}) - \gamma_{T_0}| &=  O(\eta^2) \\
    \gamma_{T_0} - \frac{2}{l''(z_*)} &\geq \frac{\delta}{2}
\end{align*}

We can assume by induction that $\gamma_{t} - \frac{2}{l''(z_*)} \geq \frac{\delta}{4}$ holds, so $z_t - z_*$ is bounded away from $0$. It follows that $\hat{Z}'(z_t) \gamma_t l''(o_t) g_t$ is bounded away from $0$. Then by Lemma \ref{thm:phase_2_bound} and a similar argument to Case 2 of Lemma \ref{thm:phase_1_approx}, we conclude that $|\hat{Z}(z_{T_1}) - \gamma_{T_1} - \eta^2 \Phi(z_{T_1})| = O(\eta^4)$ holds at time $T_1 = T_0 + O(\log{\frac{1}{\eta}})$ and
\begin{align*}
    \gamma_{T_1} &\geq \gamma_{T_0} - \sum_t |\gamma_{t+2} - \gamma_t| \\
    &= \frac{2}{l''(z_0)} + \frac{\delta}{2} - O(\eta^2 \log{\frac{1}{\eta}}) \\
    &\geq \frac{2}{l''(z_0)} + \frac{\delta}{4}
\end{align*}
This proves the lemma.
\end{proof}

\begin{lemma} \label{thm:phase_2_end_approx}
Under the conditions of \cref{thm:final_sharpness}, there exists $T_2$ such that the following hold
\begin{align}
    \label{eq:phase_2_end_approx}
    |\hat{Z}(z_{T_2}) - \gamma_{T_2} - \eta^2 \Phi(z_{T_2})| &=  O(\eta^{\frac{8}{3}}) \\
    \nonumber
    |z_{T_2} - z_*| &= O(\eta^{\frac{4}{3}})
\end{align}
\end{lemma}
\begin{proof}
By Lemma \ref{thm:phase_2_approx}, we can find $T_1$ such that
\begin{align*}
    |\hat{Z}(z_{T_1}) - \gamma_{T_1} - \eta^2 \Phi(z_{T_1})| \leq D \eta^{\frac{8}{3}}
\end{align*}
Now applying Lemma \ref{thm:phase_2_bound}, we have
\begin{align*}
    |\hat{Z}(z_{t+2}) - \gamma_{t+2} - \eta^2 \Phi(z_{t+2})| &\leq |(1 - \hat{Z}'(z_t) \gamma_t l''(o_t) g_t) (\hat{Z}(z_{t}) - \gamma_t + \eta^2 \Phi(z))| + D_1 \eta^4 g_t \\
    &\leq (1 - \hat{Z}'(z_t) \gamma_t l''(o_t) g_t)) D \eta^{\frac{8}{3}} + D_1 \eta^4 |g_t|
\end{align*}
Now we can choose $C$ large enough such that while $|z_t - z_*| \geq C \eta^{\frac{4}{3}}$, we have
$(\hat{Z}'(z_t) \gamma_t l''(o_t) g_t) D \eta^{\frac{8}{3}} + D_1 \eta^4 |g_t| \leq 0$, so 
\cref{eq:phase_2_end_approx} holds by induction.

Now suppose for the sake of contradiction $|z_t - z_*| \geq C \eta^{\frac{4}{3}}$ holds for all $t$. Then by \cref{eq:two_step_s_approx}, $\gamma_t$ will continue to decrease exponentially. But if $\gamma_t$ is sufficiently smaller than $\frac{2}{l''(z_*)}$, then \cref{eq:phase_2_end_approx} cannot hold. We conclude that there exists some time $T_2$ where both the stated conditions are satisfied.
\end{proof}

\subsection{Phase III}
From Phase II, we have tracked the GD iterates until very close to the minimum. We show that from this point onwards the iterates converge linearly.

\begin{lemma} \label{thm:phase_3_convergence}
Under the conditions of \cref{thm:final_sharpness}, $z_t$ converges to $z_*$ and
the $\gamma_t$ converges to $\frac{2}{l''(z_*)} - \eta^2 \Phi(z_*) + O(\eta^{\frac{8}{3}})$.
\end{lemma}

\begin{proof}
By Lemma \ref{thm:phase_2_end_approx}, there exists time $T_2$ such that
\begin{align*}
    |\hat{Z}(z_{T_2}) - \gamma_{T_2} - \eta^2 \Phi(z_{T_2})| &=  O(\eta^{\frac{8}{3}}) \\
    |z_{T_2} - z_*| &= O(\eta^{\frac{4}{3}})
\end{align*}

Suppose at some time $t$ the above two conditions are satisfied. By Lemma \ref{thm:z_hat_definition}, we have $|\hat{Z}(z_t) - \hat{Z}(z_*)| = O(|z_t - z_*|^2) = O(\eta^{\frac{8}{3}})$. And by Lemma \ref{thm:phi_definition} we have $|\Phi(z_t) - \Phi(z_*)| = O(|z_t - z_*|) = O(\eta^{\frac{4}{3}})$. Therefore
\begin{align*}
    \gamma_t &= \hat{Z}(z_t) - \eta^2 \Phi(z_t) + O(\eta^{\frac{8}{3}}) \\
    &= \hat{Z}(z_*) - \eta^2 \Phi(z_*) + O(\eta^{\frac{8}{3}}) \\
    &= \frac{2}{l''(z_*)} - \eta^2 \Phi(z_*) + O(\eta^{\frac{8}{3}})
\end{align*}

Now a Taylor expansion of $l'$ gives
\begin{align*}
    l'(z_t) = l''(z_*) (z_t - z_*) + O(|z_t - z_*|^2)
\end{align*}

Assume by induction that $\hat{Z}(z_t) - \gamma_t - \eta^2 \Phi(z_*) = O(\eta^{\frac{8}{3}})$. Then from \cref{eq:two_step_z_approx}, we have
\begin{align*}
    z_{t+2} - z_t &= -\gamma_t l''(\theta_t) g_t (\hat{Z}(z_t) - \gamma_t) + O(\eta^2 g_t^2) \\
    &= -\gamma_t l''(\theta_t) l''(z_*)(z_t - z_*) \eta^2 \Phi(z_*) + O(\eta^{\frac{8}{3}} |z_t - z_*| ) + O(\eta^2 |z_t - z_*|^2) \\
    z_{t+2} - z_* &= (1 - \eta^2\gamma_t l''(\theta_t) l''(z_*) \Phi(z_*))(z_t - z_*) + O(\eta^{\frac{8}{3}} |z_t - z_*| ) + O(\eta^2 |z_t - z_*|^2)
\end{align*}
For sufficiently small $\eta$, there exists constant $C$ such that
\begin{align*}
    z_{t+2} - z_* &= (1 - C \eta^2)(z_t - z_*)
\end{align*}
It follows that $z_t$ converges linearly to $z_*$.

Finally, from \cref{eq:two_step_gamma_approx_refined} we have
\begin{align*}
    \gamma_{T} - \gamma_{T_2} &\leq \sum_t |\gamma_{t+2} - \gamma_t| \\
    &\leq \sum_t C_1 \eta^2 |z_t - z_*|^2 + C_2 \eta^4 |z_t - z_*| \\
    &\leq \sum_t C_1 \eta^2 ((z_{T_2} - z_*)(1 - C\eta^2))^{2t} +  C_2 \eta^4 ((z_{T_2} - z_*)(1 - C\eta^2))^{t} \\
    &= O(|z_{T_2} - z_*|^2) + O(\eta^2 |z_{T_2} - z_*|) \\
    &= O(\eta^{\frac{8}{3}}).
\end{align*}
This completes the proof

\end{proof}

\subsection{Proof of \cref{thm:final_sharpness}}
We now show \cref{thm:final_sharpness} as an immediate consequence of Lemma \ref{thm:phase_3_convergence}

\begin{proof}
By Lemma \ref{thm:phase_3_convergence}, at convergence we have
\begin{equation*}
    \eta s = \frac{2}{l''(z_*)} - \eta^2 \Phi(z_*) + O(\eta^{{\frac{8}{3}}}).
\end{equation*}
Rearranging gives
\begin{equation*}
    s l''(z_*) = \frac{2}{\eta} - \eta l''(z_*) \Phi(z_*) + O(\eta^{{\frac{5}{3}}})
\end{equation*}
By Lemma \ref{thm:sharpness_formula}, this is precisely the sharpness at the local minimum. Substituting the value of $\Phi(z_*)$ gives the desired result.
\end{proof}

\section{Proof of Lemmas in \cref{sec:discussion}}

\subsection{Proof of Lemma \ref{thm:bce}}
\begin{proof}
Let $s = \sigma(z)$. Then
\begin{align*}
    \text{BCE}_q'(z) &= s - q \\
    \text{BCE}_q''(z) &= s(1-s) \\
    \text{BCE}_q^{(3)}(z) &= s(1-s)(1-2s)\\
    \text{BCE}_q^{(4)}(z) &= s(1-s)(1 - 6s + 6s^2) \\
\end{align*}
Then
\begin{equation*}
    \alpha_{\text{BCE}_q}(z) = 2 s^2 (1 - s)^2 (3 s^2 - 3s + 1) > 0.
\end{equation*}
\end{proof}

\subsection{Proof of Lemma \ref{thm:mlsq}}
\begin{proof}
We have
\begin{align*}
    \text{MLSq}_n'(z) &= 2(z^n - a)(nz^{n-1}) \\
    \text{MLSq}_n''(z) &= 2n[(n-1)z^{n-2} (z^n - a) + n z^{2n-2} ] \\
    \text{MLSq}_n^{(3)}(z) &= 2n[(n-1)(n-2)z^{n-3}(z^n - a) + 3n(n-1)z^{2n-3}] \\
    \text{MLSq}_n^{(4)}(z) &= 2n[(n-1)(n-2)(n-3) z^{n-4} (z^n - a) + (n-1)(7n - 6) z^{2n-4}] \\
\end{align*}
Set $z_* = a^{\frac{1}{n}}$. Then
\begin{align*}
    \alpha_{\text{MLSq}_{a,n}}(z_*) &= 3\left(6n^2(n-1) z_*^{2n-3}\right)^2 - (2n^2 z_*^{2n-2})(2n(n-1)(7n-6) z_*^{2n-4} \\
    &= 4n^3(n-1)(27n^2 - 34n + 6) x_0^{4n - 6} \\
    &> 0
\end{align*}
\end{proof}

\subsection{Proof of Lemma \ref{thm:dor}}
\begin{proof}
Since scaling the loss by a constant does not affect the product-stability, WLOG assume $C_a = 1$. 

Let $g(z) = \log(e^{z} + 1) + \log(e^{-z} + 1)$, such that $l_a(z) = (g(z - 1))^a$. Denoting $u = e^z, v = e^{-z}$
\begin{align*}
    g'(z) &= \frac{u}{1+u} - \frac{v}{1 + v} \\
    g''(z) &= \frac{u}{(1+u)^2} + \frac{v}{(1+v)^2} \\
    g^{(3)}(z) &= \frac{u(1-u)}{(1+u)^3} - \frac{v(1-v)}{(1+v)^3} \\
    g^{(4)}(z) &= \frac{u(1-4u+u^2)}{(1+u)^4} - \frac{v(1-4v+v^2)}{(1+v)^4}
\end{align*}
Thus we have $g(0) = 2\log{2}, g'(0) = 0, g''(0) = \frac{1}{2} g^{(3)}(z) = 0, g^{(4)}(z) = -\frac{1}{4}$.

Now another calculation shows
\begin{align*}
    f'(z) &= a\, g(z-1)^{a-1}\, g'(z-1) \\
    f''(z) &= a\Big[(a-1)\,g(z-1)^{a-2}\,(g'(z-1))^2 \;+\; g(z-1)^{a-1}\,g''(z-1)\Big] \\
    f^{(3)}(z) &= a\Big[
    (a-1)(a-2)\,g(z-1)^{a-3}\,(g'(z-1))^3 \\
    &+ 3(a-1)\,g(z-1)^{a-2}\,g'(z-1)\,g''(z-1) \\
    &+ g(z-1)^{a-1}\,g^{(3)}(z-1)
    \Big] \\
    f^{(4)}(z) &= a\Big[
    (a-1)(a-2)(a-3)\,g(z-1)^{a-4}\,(g'(z-1))^4 \\
    & + 6(a-1)(a-2)\,g(z-1)^{a-3}\,(g'(z-1))^2\,g''(z-1) \\
    & + 3(a-1)\,g(z-1)^{a-2}\,(g''(z-1))^2 \\
    & + 4(a-1)\,g(z-1)^{a-2}\,g'(z-1)\,g^{(3)}(z-1) \\
    & + g(z-1)^{a-1}\,g^{(4)}(z-1)
    \Big]
\end{align*}
Then
\begin{align*}
    f''(1) &= \frac{a}{2}(2 \log{2})^{a-1} \\
    f^{(3)}(1) &= 0 \\
    f^{(4)}(1) &= \frac{a}{4}(-(2\log{2})^{a-1} + 3(a-1)(2\log{2})^{a-2})
\end{align*}

We conclude that
\begin{align*}
    \alpha_{l_a}(1) = \frac{a^2}{8} (2 \log{2})^{2a-3} (2 \log{2} - 3(a-1)) > 0
\end{align*}
\end{proof}

\section{Experimental Setup} \label{apdx:experimental_setup}
\subsection{Calculating the Multivariate Product Stability}
We provide the following code for calculating the multivariate product stability.

\lstset{
    language=Python,
    basicstyle=\ttfamily\small,
    keywordstyle=\color{blue},
    commentstyle=\color{gray},
    stringstyle=\color{green!60!black},
    showstringspaces=false,
    breaklines=true,       
    frame=single
}
\begin{lstlisting}

def get_hvp(closure, params, vec, create_graph=False):
    loss = closure()

    grads = torch.autograd.grad(loss, params, create_graph=True)
    flat_grads = torch.cat([g.reshape(-1) for g in grads])

    hvp = torch.autograd.grad(
        flat_grads,
        params,
        grad_outputs=vec,
        create_graph=create_graph
    )

    hvp_flat = torch.cat([h.reshape(-1) for h in hvp])

    return hvp_flat if create_graph else hvp_flat.detach()

def pseudo_inverse_cg(closure, params, b, iters=50, tol=1e-8):
    """
    Computes x = H^+ b using CGLS.
    This gives the minimum-norm solution even if H is singular.
    """
    x = torch.zeros_like(b)

    # r = b - Hx
    r = b.clone()

    # s = H^T r (Hessian is symmetric)
    s = get_hvp(closure, params, r, create_graph=False)

    p = s.clone()
    norm_s_old = torch.dot(s, s)

    for _ in range(iters):
        Hp = get_hvp(closure, params, p, create_graph=False)

        denom = torch.dot(Hp, Hp) + 1e-12
        alpha = norm_s_old / denom

        x = x + alpha * p
        r = r - alpha * Hp

        # convergence check
        if torch.norm(r) < tol:
            break

        s = get_hvp(closure, params, r, create_graph=False)

        norm_s_new = torch.dot(s, s)
        beta = norm_s_new / (norm_s_old + 1e-12)

        p = s + beta * p
        norm_s_old = norm_s_new

    return x

def compute_robust_stability(model, loss_fn, data, target):
    params = [p for p in model.parameters() if p.requires_grad]

    def closure():
        logits = model(data)
        return loss_fn(logits, target)

    # ---------------------------
    # 1. Power iteration
    # ---------------------------
    v = torch.randn(sum(p.numel() for p in params))
    v /= (torch.norm(v) + 1e-9)

    for _ in range(30):
        hv = get_hvp(closure, params, v, create_graph=False)
        v = hv / (torch.norm(hv) + 1e-9)

    # ---------------------------
    # higher-order
    # ---------------------------
    loss = closure()

    grads = torch.autograd.grad(loss, params, create_graph=True)
    flat_grads = torch.cat([g.reshape(-1) for g in grads])

    hvp_v = torch.autograd.grad(
        flat_grads, params, grad_outputs=v, create_graph=True
    )
    hvp_v_flat = torch.cat([h.reshape(-1) for h in hvp_v])

    phi = torch.dot(hvp_v_flat, v)

    # g3
    g3_list = torch.autograd.grad(phi, params, create_graph=True)
    g3 = torch.cat([g.reshape(-1) for g in g3_list])

    # ---------------------------
    # 3. CG solve
    # ---------------------------
    h_inv_g3 = pseudo_inverse_cg(closure, params, g3.detach())

    term1 = 3 * torch.dot(g3.detach(), h_inv_g3)

    # ---------------------------
    # 4. 4th derivative
    # ---------------------------
    d4_list = torch.autograd.grad(torch.dot(g3, v), params)
    d4 = torch.dot(torch.cat([d.reshape(-1) for d in d4_list]), v)

    return (term1 - d4).item()

\end{lstlisting}

\begin{figure}[b] 
  \centering
  \includegraphics[width=0.48\textwidth]{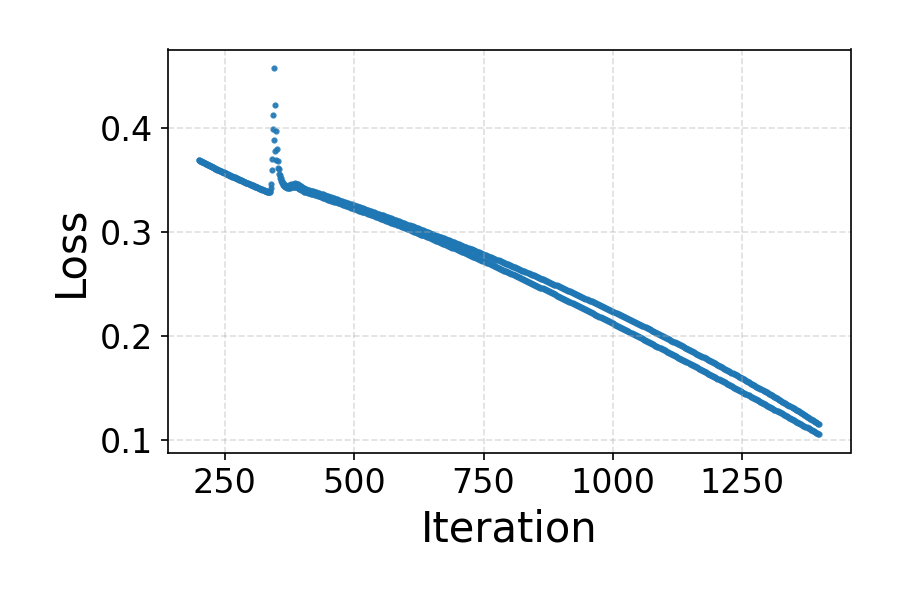}
  \caption{Loss Oscillation in EoS Regime. Zoomed in version of \cref{fig:a}, showing that the loss is oscillating in the EoS regime.}
  \label{fig:oscillation}
\end{figure}

\end{document}